\documentclass[pdflatex,sn-apa]{sn-jnl}

\usepackage{placeins} 
\usepackage{subcaption} 
\usepackage{enumitem} 
\usepackage{csquotes} 
\usepackage{tikz}
\usetikzlibrary{
    graphs,                   
    arrows.meta,              
    decorations.markings,      
    knots,                    
    calc,                     
    fit,                      
    decorations.pathmorphing,  
    tikzmark,                 
    backgrounds               
}

\usepackage{graphicx}%
\usepackage{multirow}%
\usepackage{amsmath,amssymb,amsfonts}%
\usepackage{amsthm}%
\usepackage{mathrsfs}%
\usepackage[title]{appendix}%
\usepackage{xcolor}%
\usepackage{textcomp}%
\usepackage{manyfoot}%
\usepackage{booktabs}%
\usepackage{algorithm}%
\usepackage{algorithmicx}%
\usepackage{algpseudocode}%
\usepackage{listings}%

\newcommand{\Exp}[1]{\ensuremath{\text{E} \left[ #1 \right]}}
\newcommand{\Empexp}[1]{\ensuremath{\text{E}_n \left[ #1 \right]}}

\newcommand{\indep}{\ensuremath{\mbox{\(\perp \!\!\!\! \perp\)}}}
\newcommand{\argmin}[1]{ \ensuremath{ \underset{#1}{\operatorname{argmin}}\; } }
\newcommand{\pder}[2]{\frac{\partial #1}{\partial #2}}

\newcommand{\yyhat}{\ensuremath{\widehat{y}}}
\newcommand{\yhat}{\ensuremath{\widehat{Y}}}
\newcommand{\yhatu}{\ensuremath{\widehat{Y}_{\hat{\theta}_\text{U}}}}
\newcommand{\yhatft}{\ensuremath{\widehat{Y}_{\hat{\theta}_\text{FT}}}}
\newcommand{\yhatspt}{\ensuremath{\widehat{Y}_{\hat{\theta}_\text{SPT}}}}
\newcommand{\ls}{\ensuremath{\lambda_\text{SDP}}}
\newcommand{\lp}{\ensuremath{\lambda_\text{PPD}}}

\newcommand{\scm}{\ensuremath{\mathcal{M}}}
\newcommand{\scmdef}{\ensuremath{\mathcal{M} = \langle \mathbf{U}, \mathbf{V}, \mathbf{F} \rangle}}
\newcommand{\scmhat}{\ensuremath{\mathcal{M}_{\yhat}}}
\newcommand{\scmhatdef}{\ensuremath{\mathcal{M}_{\yhat} = \langle \widetilde{\mathbf{U}}, \widetilde{\mathbf{V}}, \widetilde{\mathbf{F}} \rangle}}
\newcommand{\diagram}{\ensuremath{\mathcal{G}}}
\newcommand{\parents}{\ensuremath{\mathbf{PA}}}
\newcommand{\pparents}{\ensuremath{\mathbf{pa}}}
\newcommand{\pathset}{\ensuremath{\Pi}}
\newcommand{\dashedbiarrow}{
  \mathrel{
    \vcenter{
      \hbox{\(\dashleftarrow\)}
    }
    \mkern-18mu
    \vcenter{
      \hbox{\(\dashrightarrow\)}
    }
  }
}
\tikzset{
    neutral node/.style={draw, circle, fill=black, inner sep=0pt, minimum size=1.5mm},
    unfair node/.style={neutral node, BrickRed, fill=BrickRed},
    fair node/.style={neutral node, ForestGreen, fill=ForestGreen},
    neutral edge/.style={>={Latex[length=1.2mm,width=1.2mm]}, line width=0.5pt, shorten >=3pt, shorten <=3pt},
    fair edge/.style={neutral edge, ForestGreen},
    unfair edge/.style={neutral edge, BrickRed, decorate, decoration={snake, amplitude=.2mm, segment length=1mm, post length=2.4mm, pre length=1mm}},
    between edge/.style={>={Latex[length=2.5mm,width=2.5mm]}, line width=1.5pt, shorten >=5pt, shorten <=5pt}
}
\definecolor{cbDeepBlue}{HTML}{1170aa}
\definecolor{cbBrightOrange}{HTML}{fc7d0b}
\definecolor{cbSlateGray}{HTML}{a3acb9}
\definecolor{cbCharcoal}{HTML}{57606c}
\definecolor{cbSkyBlue}{HTML}{5fa2ce}
\definecolor{cbDarkOrange}{HTML}{c85200}
\definecolor{cbSoftGray}{HTML}{7b848f}
\definecolor{cbPaleBlue}{HTML}{a3cce9}
\definecolor{cbPeach}{HTML}{ffbc79}
\definecolor{cbLightGray}{HTML}{c8d0d9}
\definecolor{cbBlue}{named}{cbDeepBlue}
\definecolor{cbOrange}{named}{cbBrightOrange}
\definecolor{cbCharcoal}{named}{cbCharcoal}
\definecolor{cbPaleBlue}{named}{cbPaleBlue}
\definecolor{comment}{named}{cbDarkOrange}
\definecolor{BrickRed}{HTML}{b6321c}
\definecolor{ForestGreen}{HTML}{009b55}
\definecolor{NavyBlue}{HTML}{000080}
\definecolor{PredictorBlue}{HTML}{0080CF}
\definecolor{SelectionYellow}{HTML}{FCA311}
\definecolor{FairGreen}{HTML}{328B48}
\definecolor{UnfairRed}{HTML}{D41715}
\definecolor{BrickRed}{HTML}{b6321c}

\theoremstyle{thmstyleone}%
\newtheorem{theorem}{Theorem}
\newtheorem{proposition}{Proposition}

\theoremstyle{thmstyletwo}%
\newtheorem{remark}{Remark}%

\theoremstyle{thmstylethree}%
\newtheorem{definition}{Definition}%

\usepackage{fancyhdr}
\fancypagestyle{firstpage}{
    \fancyhf{}
    
    \renewcommand{\footrule}{%
        \hbox to\headwidth{\rule{37mm}{0.2mm}}
        \vskip\footruleskip}
    \fancyfoot[L]{\small Published in \textit{Machine Learning} (Springer) on May 18, 2026 \\DOI: \href{https://doi.org/10.1007/s10994-026-07061-7}{10.1007/s10994-026-07061-7}}
}

\begin{document}
\thispagestyle{firstpage}

\title[Tuning Derivatives for Causal Fairness in Machine Learning]{Tuning Derivatives for Causal Fairness in Machine Learning}


\author*[1]{\fnm{Filip} \sur{Edstr\"om}}\email{filip.edstrom@umu.se}

\author[2]{\fnm{Guilherme W. F.} \sur{Barros}}\email{guilherme.barros@umu.se}

\author[1]{\fnm{Tetiana} \sur{Gorbach}}\email{tetiana.gorbach@umu.se}

\author[1]{\fnm{Xavier} \spfx{de} \sur{Luna}}\email{xavier.de.luna@umu.se}

\affil*[1]{\orgdiv{Department of Statistics, Ume\aa\ School of Business, Economics and Statistics}, \orgname{Ume\aa\ University}, \orgaddress{\city{Ume\aa}, \postcode{901 87}, \country{Sweden}}}
\affil[2]{\orgdiv{Integrated Science Lab, Department of Physics}, \orgname{Ume\aa\ University}, \orgaddress{\city{Ume\aa}, \postcode{901 87}, \country{Sweden}}}

\abstract{
    Artificial-intelligence systems are becoming ubiquitous in society, yet their predictions typically inherit biases with respect to protected attributes such as race, gender, or age. Classical fairness notions, most notably Statistical Parity (SP), demand that predictions be independent of the protected attributes, but are overly restrictive when these attributes influence mediating variables that are considered business necessities. Recent causal formulations relax SP by distinguishing allowed from not-allowed causal paths and by complementing SP with Predictive Parity (PP), requiring the predictor to replicate the legitimate influence of business-necessities. Existing path-based definitions are mainly practical when applied to categorical attributes. This paper introduces a new framework for fairness in structural causal models that is tailored to continuous protected attributes. We formalize SP and PP through path-specific partial derivatives, establish conditions under which these criteria coincide with prior causal definitions, and characterize when a fair predictor, one that satisfies SP along not-allowed paths while achieving PP along allowed paths, exists. Building on this theory, we propose a fair tuning algorithm that either constructs such a predictor or, when not possible, allows for a trade-off between SP and PP. We present experiments on simulated and real data to evaluate our proposal, compare it with previously proposed methods, and show that it performs better when PP is considered.
}


\maketitle

\section{Introduction}
\label{sec:introduction}

Artificial-Intelligence (AI) systems are increasingly being incorporated into decision-making processes. Yet, AI systems have been criticized for replicating discriminatory behaviors observed in society, for example, racial bias in the COMPAS recidivism risk assessment, gender bias in Amazon's hiring system, and bias with respect to age \citep{stypinska_ai_2023}. \footnote{ ''Machine Bias---Risk assessments in criminal sentencing," \textit{ProPublica},  \url{https://www.propublica.org/article/machine-bias-risk-assessments-in-criminal-sentencing}, accessed June 18, 2025.} \footnote{ ''Amazon scrapped 'sexist AI' tool," \textit{BBC},  \url{https://www.bbc.com/news/technology-45809919}, accessed June 18, 2025.} When regulation or ethical aspects do not allow replicating bias with respect to a feature, for example, age, then we call such a feature a ``protected attribute". Fairness in machine learning is the field devoted to defining metrics to evaluate biases with respect to protected attributes and to finding ways to develop AI systems that avoid or minimize biases existing in the data used to create these systems. 

Consider situations where a data-driven predictor $\yhat$ for a variable of interest $Y$ is obtained using data for which one or several protected attributes $\mathbf{X}$ are related to $Y$. Several definitions of what constitutes fair predictors have been proposed, see \cite{barocas_fairness_2023} for an overview. An early and intuitive fairness condition, Statistical Parity (SP)  \citep{darlington_another_1971}, requires that predictions be independent of protected attributes, that is, $\yhat \indep \mathbf{X}$. However, this condition may be too strict if the protected attributes depend on other attributes that are deemed necessary for prediction, called ``business necessities" in \cite{plecko_reconciling_2024} and ``resolving variables" in \cite{kilbertus_avoiding_2017}. As an example, consider a recidivism risk assessment system, such as COMPAS, which produces a risk score $\yhat$ indicating the probability of recidivism for an offender based on the features age, $A$, race, $R$, and the number of prior offenses, $P$, all dependent on each other. Let $R$ and $A$ be protected attributes, while $P$ is a business necessity. In this case, our ideal goal would be to design a prediction algorithm $\yhat$ that relies on $P$, but not on $A$ and $R$ (see Figure \ref{fig:problem}). However, satisfying SP in $A$ and $R$ when they are related to $P$ is not possible without excluding $P$. To offer a solution to this problem and allow for more flexibility than SP, we need to consider how the attributes are related to each other and $Y$. 

\begin{figure}
    \centering
    \begin{tikzpicture}
        \begin{scope}[shift={(0,0)}]
            \node (DGP) at (0, -2) {Unfair process};
            
            \node[unfair node] (Adgp) at (-0.8, 0) {};
            \node[unfair node] (Rdgp) at (0, 0.8) {};
            \node[neutral node] (Pdgp) at (0, -0.8) {};
            \node[unfair node] (Ydgp) at (0.8, 0) {};
    
            \node[left] at (Adgp.north) {$A$};
            \node[above] at (Rdgp.west) {$R$};
            \node[below] at (Pdgp.south) {$P$};
            \node[right] at (Ydgp.east) {$Y$};
    
            \draw[->, fair edge] (Rdgp) to (Pdgp);
            \draw[->, unfair edge] (Rdgp) to (Ydgp);
            \draw[<->, neutral edge, dashed] (Adgp) to[bend left] (Rdgp);
            \draw[->, fair edge] (Adgp) to (Pdgp);
            \draw[->, unfair edge] (Adgp) to (Ydgp);
            \draw[->, fair edge] (Pdgp) to (Ydgp);
        \end{scope}
    
        \begin{scope}[shift={(4,0)}]
            \node (Data) at (0, -2) {Data}; 
            \node[below] (Table) at (0,1) {
                \small
                \begin{tabular}{@{}c@{}c@{}c@{}c@{}}
                    $A$      &$R$        & $P$        &  $Y$ \\
                    \hline                 
                    $a_1$    &$r_1$      & $p_1$      &  $y_1$    \\
                    $\vdots$ &$\vdots$   & $\vdots$   &  $\vdots$ \\ 
                    $a_n$    &$r_n$      & $p_n$      &  $y_n$    \\
                    \hline
                \end{tabular}
            };
        \end{scope}
    
        \begin{scope}[shift={(8,0)}]
            \node (Hat) at (0, -2) {Fair predictor};
            
            \node[unfair node] (Ahat) at (-0.8, 0) {};
            \node[unfair node] (Rhat) at (0, 0.8) {};
            \node[neutral node] (Phat) at (0, -0.8) {};
            \node[fair node] (Yhat) at (0.8, 0) {};
    
            \node[left] at (Ahat.north) {$A$};
            \node[above] at (Rhat.west) {$R$};
            \node[below] at (Phat.south) {$P$};
            \node[right] at (Yhat.east) {$\yhat$};
    
            \draw[->, fair edge] (Rhat) to (Phat);
            \draw[<->, neutral edge, dashed] (Ahat) to[bend left] (Rhat);
            \draw[->, fair edge] (Ahat) to (Phat);
            \draw[->, fair edge] (Phat) to (Yhat);
        \end{scope}
    
        \draw[->, between edge] (0.5, 1) to[bend left=15] node[above] {Observe} (3, 1);
        \draw[->, between edge] (5, 1) to[bend left=15] node[above] {Learn} (7.5, 1);
    \end{tikzpicture}
    \caption{The problem of creating a fair predictor from data. We observe data from a biased process, with not-allowed paths (red, squiggly arrows) from the protected attributes $A, R$ to the outcome $Y$. Fair prediction removes the not-allowed path, but maintains the allowed paths (green, straight arrows).}
    \label{fig:problem}
\end{figure}
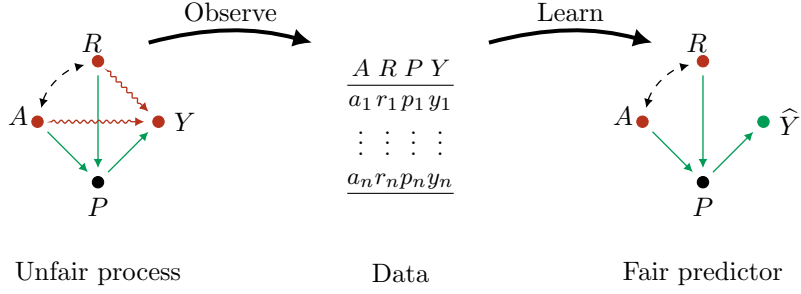

For this purpose, we use Structural Causal Models (SCMs) \citep{pearl_causality_2009}. Such models allow us to decompose the way protected attributes relate to $Y$ into causal paths \citep{zhang_causal_2017, kilbertus_avoiding_2017,  kusner_counterfactual_2017, nabi_fair_2018, zhang_fairness_2018, chiappa_path-specific_2019, wu_pc-fairness_2019, plecko_causal_2024}. In our example, this means considering the paths that represent the direct effect of $A$ and  $R$ on $Y$, illustrated $A \to Y, R \to Y$, as \textit{not-allowed} paths and the ones that represent indirect effects through $P$, $A \to P \to Y$ and $R \to P \to Y$, as \textit{allowed} paths, see Figure \ref{fig:problem}. Recently, \cite{plecko_reconciling_2024} not only constrained how the not-allowed paths are modeled, but also how the allowed paths are modeled through the concept of Predictive Parity (PP) \citep{chouldechova_fair_2017}. For example, PP with respect to $P$ states that the predictor $\yhat$ should capture all the variance in $Y$ due to $P$, that is, $Y \indep P \mid \yhat$. \cite{plecko_reconciling_2024} define SP and PP in terms of contrasts, that is, the change in a variable of interest due to a discrete change in another variable, and suggests to constrain predictors to satisfy such a causal SP along contrasts over not-allowed paths and causal PP over contrasts along allowed paths. Common to most of the fairness literature based on causal contrasts is that it is most suitable for categorical protected attributes such as gender or race. However, when considering continuous protected attributes, such as age, causal contrasts must be evaluated at many points, which is impractical. Two exceptions presenting definitions that are readily applicable to continuous protected attributes are given in \cite{lindholm_discrimination-free_2022} and \cite{kancheti_matching_2022}, although these do not explicitly consider predictive parity.

In this paper, we contribute as follows: We define SP and PP in terms of partial derivatives along paths to specifically address continuous attributes. We show under which conditions our definitions imply previous causal definitions in \cite{plecko_reconciling_2024} and define a fair predictor in terms of our concepts of SP and PP along allowed and not-allowed paths, respectively. We formulate a necessary and sufficient condition under which such a fair predictor exists and propose a tuning algorithm that constructs such a fair predictor when it exists. When the latter does not exist, the algorithm offers a flexible compromise between statistical and predictive parity. 

The rest of the paper is structured as follows. In Section \ref{sec:theory} we provide relevant background theory, define our novel concepts of statistical and predictive parity in terms of partial derivatives, and provide theoretical properties. In Section \ref{sec:method} we introduce an algorithm tuning SP and PP for fair prediction. We study this algorithm on neural network predictors in Section \ref{sec:numerical_experiments} using numerical experiments. We also give comparisons with earlier methods proposed in \cite{lindholm_discrimination-free_2022} and \cite{kancheti_matching_2022}. In Section \ref{sec:real_world_example}, we use the COMPAS dataset to illustrate the use of the introduced fair tuning algorithm  and compare our novel concepts of SP and PP with existing concepts \citep{plecko_reconciling_2024}.  Finally, Section \ref{sec:conclusion} concludes the paper. 

\section{Theory}
\label{sec:theory}
Following prior work on causal fairness
\citep{zhang_causal_2017, kilbertus_avoiding_2017,  nabi_fair_2018, zhang_fairness_2018,  kusner_counterfactual_2017, wu_pc-fairness_2019}, we adopt SCMs \citep{pearl_causality_2009} as our formal framework.
 
\begin{definition}[Structural Causal Model,  Def. 7.1.1 in \citealp{pearl_causality_2009}]
    \label{def:scm}
    A Structural Causal Model (SCM) is a triple 
    \[\scmdef\]
    where:
    \begin{enumerate}[label=\roman*.]
        \item $\mathbf{U}$ is a set of background variables, that are determined by factors outside the model;
        \item $\mathbf{V}$ is a set $\{V_1, V_2, ..., V_n\}$ of variables, called \textit{endogenous}, that are determined by variables in the model---that is, variables in $\mathbf{U} \cup \mathbf{V}$; and
        \item $\mathbf{F}$ is a set of functions (mechanisms) $\{f_1, f_2, ..., f_n\}$ such that each $f_i$ is a mapping from (the respective domains of) $U_i \cup \parents_i$ to $V_i$, where $U_i \subseteq \mathbf{U}$ and $\parents_i \subseteq \mathbf{V} \setminus V_i$, called parents of $V_i$, and the entire set $\mathbf{F}$ forms a mapping from $\mathbf{U}$ to $\mathbf{V}$. In other words, each $f_i$ in 
        \[
        v_i = f_i(\pparents_i, u_i), \quad i = 1, ..., n
        \]
        assigns a value to $V_i$ that depends on (the values of) a subset of variables in $\mathbf{V} \cup \mathbf{U}$, and the entire set $\mathbf{F}$ has a unique solution $\mathbf{V}(\mathbf{u})$.
    \end{enumerate}
\end{definition}

We consider probabilistic SCMs, with a probability distribution assigned to $\mathbf{U}$. An SCM $\scm$ can be attributed with a causal diagram, $\diagram$, that describes the relationships between the variables in $\scm$ such that, if $V_i \in \parents_j$, then the arrow $V_i \to V_j$ is in $\diagram$. $V_i \to V_j$ is an example of a directed path of length $|\pi| = 1$. Formally, a directed path $\pi$ is an ordered sequence of variables, $\pi = (V_1, ..., V_{|\pi|})$, where every variable is a parent of the next variable in the sequence. We say that $V \in \pi$ if $V$ appears on the path $\pi$, and if there is a directed path from $V_i$ to $V_k$ then $V_i$ is called an ancestor of $V_k$, and conversely $V_k$ is called a descendant of $V_i$. When considering multiple paths, we use $\pathset$ to denote a set of paths. Additionally, if the background variables related to two endogenous variables $V_i, V_j$ are correlated, that is, $U_i$ and $U_j$ are not independent, this may be illustrated with a dashed, double-headed arrow, $V_i \dashedbiarrow V_j$.

SCMs allow us to model interventions, where an intervention on a variable $V_i$ replaces the mechanism $f_i$ according to the intervention \citep{pearl_causality_2009}. For example, the intervention of assigning the value $v_i$ to the variable $V_i$ replaces $V_i = f_i(\parents_i, U_i)$ with $V_i = v_i$. We use the potential outcome notation \citep{rubin_estimating_1974}, $V_j(v_i)$, to describe the counterpart of $V_j$ under the intervention $V_i = v_i$.


Using the potential outcome notation and paths, we can now define a path-specific effect of a change in the root node of a path, $V_1$, from the value $v_1$ to the value $v_1^*$ on the target of the path, $V_{|\pi|}$. The path-specific effect is then the difference in $V_{|\pi|}$ when the change in $V_1$ is allowed to propagate only over the path $\pi$. This is formalized using the potential outcome notation and nested counterfactuals in \cite{shpitser_counterfactual_2013}. For instance, for an SCM with diagram $X \to Y, X \to W \to Y$ we can write $Y(W(x), X)$ to represent the intervention of setting $X = x$ only along the path $\pi: X \to W \to Y$, not affecting the value of $X$ along $X \to Y$. For convenience, we use the notation $Y(\pi(x)) := Y(W(x), X)$. A path-specific effect in a setting $\mathbf{Z}$ is equivalent to a counterfactual contrast with respect to a specific path and a change from $v_1$ to $v_1^*$ \citep{plecko_causal_2024}, that is,
\begin{equation*}
    P(V_{|\pi|}(\pi(v_1^*)) \mid \mathbf{Z}) - 
    P(V_{|\pi|}(\pi(v_1)) \mid \mathbf{Z}).
\end{equation*}
A special case of path-specific effect is the direct effect, where the path only consists of one edge, for example, $V_i \to V_j$.

\subsection{Causal Fairness Notions}

\cite{plecko_reconciling_2024} define Causal SP for a predictor $\yhat$ in terms of three sets of paths: direct (Ctf-DE), indirect (Ctf-IE), and spurious (Ctf-SE) in a specific SCM, the standard fairness model \citep{plecko_causal_2024}. 

\begin{definition}[Causal SP, Def. 4 in \citealp{plecko_reconciling_2024}]
    \label{def:causal-SP}
    Let $\pathset_{\text{IE}}$ be the set of all indirect directed paths from $X$ to $\yhat$, that is, $|\pi_i| > 2 \ \forall \pi_i \in \pathset_{IE}$. 
    Then, $\yhat$ satisfies causal statistical parity with respect to a change from $x_0$ to $x_1$ in the protected attribute $X$ if the direct, indirect, and spurious effects are zero:
    \begin{align*}
        &\underbrace{P(\yyhat(x_1, \pathset_{\text{IE}}(x_0)) \mid x_0) -  P(\yyhat(x_0, \pathset_{\text{IE}}(x_0)) \mid x_0)}_{\text{Ctf-DE}}\\
        &=
        \underbrace{P(\yyhat(x_1, \pathset_{\text{IE}}(x_0)) \mid x_0) -  P(\yyhat(x_1, \pathset_{\text{IE}}(x_1)) \mid x_0)}_{\text{Ctf-IE}} \\
        &= 
        \underbrace{P(\yyhat(x_1) \mid x_0) -  P(\yyhat(x_1) \mid x_1)}_{\text{Ctf-SE}}
        = 0.
    \end{align*}
\end{definition}

\citet{plecko_reconciling_2024} also defines Causal PP with respect to a counterfactual contrast. 

\begin{definition}[Causal PP, Def. 5 in \citealp{plecko_reconciling_2024}]
    \label{def:causal-PP}
    Let $\yhat$ be a predictor of the outcome $Y$, and let $X$ be a protected attribute. Then, $\yhat$ is said to satisfy causal predictive parity with respect to a counterfactual contrast along the path $\pi: X \to ... \to Y$ with respect to a change in $X$ from $x_0$ to $x_1$ in setting $\mathbf{Z}$ if
    \begin{equation*}
        \Exp{Y(\pi(x_1)) \mid \mathbf{Z}} - \Exp{Y(\pi(x_0)) \mid \mathbf{Z}}
        =
        \Exp{\yhat(\pi(x_1)) \mid \mathbf{Z}} - \Exp{\yhat(\pi(x_0)) \mid \mathbf{Z}}.
    \end{equation*}
\end{definition}
The fairness of a predictive algorithm can then be defined using these two definitions, as we discuss in Section \ref{sec:fairpred} below. These definitions are most natural when $X$ is categorical, but when $X$ is continuous there is an infinite number of possible contrasts. In practice, one needs to consider the definitions above for a possibly large set of points $x_1$ for each $x$, and computation then scales quadratically with this number of points. This becomes even more cumbersome when PP needs to be evaluated along all allowed paths as we shall see below.

\subsection{Derivative Fairness Notions}
\label{sec:theory:derivative-notions}
To obtain a more practical solution for continuous features, we define SP and PP in terms of derivatives along paths. This is a natural extension to contrasts as they can be regarded as discrete derivatives, up to a division by an increment. We consider an SCM $\scmhat$ that describes the prediction algorithm, that is, the causal structure of the algorithm that outputs predictions $\yhat$ and the features used as its input. Using this SCM, we define SP in terms of derivatives by applying the chain rule of calculus along the structural equations of a path.

\begin{definition}[Statistical Parity in Derivatives, SPD]
    \label{def:derivative-sp}
    Let $\mathbf{V}$ be a set of variables, where $X \in \mathbf{V}$ is a protected attribute, $\yhat \in \mathbf{V}$ is an outcome predictor. For a given SCM $\scmhat$$ = \{\mathbf{U}, \mathbf{V}, \mathbf{F}\}$, consider a directed path $\pi: V_1 = X \to \dots \to \yhat$ where all variables along the path have compact support and each mechanism $f_{k}$ along the path is assumed to be differentiable, then $\yhat$ satisfies statistical parity in derivatives along path $\pi$ if
    \[
        \prod_{k = 2}^{|\pi|} \pder{f_{V_k}(\parents_k, u_k)}{v_{k-1}} 
        = 0 \quad \forall x, \mathbf{u},
    \]
    where $|\pi|$ is the length of the path, $V_{|\pi|} = \yhat$, $f_{V_k}$ is the mechanism assigning values to $V_{k} \in \mathbf{V}$.
\end{definition}

If SPD holds, $\yhat$ is constant in the path-specific effect along $\pi$ and at least one arrow is missing along the considered path in the diagram describing $\scmhat$. Thus, SPD along a path $\pi$ can be achieved for a prediction algorithm that generates predictions $\yhat$, by simply excluding the parent of $Y$ along $\pi$ when fitting $\yhat$, in the literature called ``Fairness Through Unawareness" (FTU) \citep{mehrabi_survey_2021}. However, this will be limiting when we later consider PP. In that case, SPD can better be achieved by constraining the derivative of $f_{\yhat}$ with respect to its parent along $\pi$ to be zero. Another option is to generate an ancestor of $\yhat$ along $\pi$, such that it has zero derivative with respect to its parent along $\pi$.\footnote{By ``generate" we mean generating a variable artificially in place of an observed variable. This creates a corresponding SCM where the variable and its mechanism are replaced by their artificial counterparts. An example is a generated prediction $\yhat$ replacing the outcome $Y$.} In this case, the descendants of the generated variable should be generated sequentially, using the generated version of their parent along $\pi$. For example, in the unfair process in Figure \ref{fig:problem}, we may want to satisfy SPD along the path $A \to P \to Y$. We can achieve this by either generating $\yhat$ such that $\frac{\partial f_{\yhat}}{\partial p} = 0$, or by generating $\widehat{P} = f_{\widehat{P}}(\parents_{\widehat{P}})$ such that $\frac{\partial f_{\widehat{P}}}{\partial a} = 0$ and use $\widehat{P}$ in sequential prediction, that is, $\yhat = f_{\yhat}(\widehat{P})$. Examples of sequential predictors are neural causal models \citep{xia_causal-neural_2021} and, in the causal fairness literature, the methods developed in \citet{plecko_fair_2020} and \citet{van_breugel_decaf_2021} generate fair $\yhat$ sequentially. Additionally, the approaches in \citet{kusner_counterfactual_2017} and \citet{chiappa_path-specific_2019} infer latent variables prior to generating $\yhat$, thereby generating other variables than the outcome.

SPD is related to Causal SP (Definition \ref{def:causal-SP}) as follows.

\begin{proposition}
    \label{prop:sp-same}
    For a given SCM $\scmhat = \langle \mathbf{U}, \mathbf{V}, \mathbf{F}\rangle$, with $X, \yhat \in \mathbf{V}$, where $X$ is a protected attribute, let $\yhat$ be an outcome predictor that satisfies SPD with respect to $X$ along all directed paths from $X$ to $\yhat$. Then, $\yhat$ satisfies $\text{Ctf-DE} = \text{Ctf-IE} = 0$ in Definition \ref{def:causal-SP} for any values $x_1, x_0$ of $X$.
\end{proposition}

\begin{proof}
    Let us first consider the direct path $\pi: X \to \yhat$. Then by SPD along $\pi$ we have that $f_{\yhat} (x, \pparents_{\yhat} \setminus x, u_{\yhat})$ is constant in $x$. From Definition 2 Ctf-DE is satisfied when 
    \[
        P(\yhat(x_1, \pathset_{\text{IE}}(x_0)) \mid x_0) 
        - 
        P(\yhat(x_0, \pathset_{\text{IE}}(x_0)) \mid x_0) = 0.
    \]
    From the definition of SCMs this is equivalent to 
    \[
        P(f_{\yhat}(x_1, \pathset_{\text{IE}}(x_0), u_{\yhat}) \mid x_0) 
        -
        P(f_{\yhat}(x_0, \pathset_{\text{IE}}(x_0), u_{\yhat}) \mid x_0) = 0,
    \]
    which holds since $f_{\yhat}$ is constant in $x$, that is, the probability function is the same for a change in $x$ along the direct path.
    
    Next, we show that $\text{Ctf-IE} = 0$ with a proof by contradiction: Let us assume that $\text{Ctf-IE} \neq 0$. Then, there exists an indirect path $\pi: X \to V_k \to \yhat$ such that its path-specific effect of a change from $x_0$ to $x_1$ is non zero, $P(\yhat(x,\pi(x_1)) \neq P(\yhat(x,\pi(x_0)),\forall x$. Then, all functions $f_{V_{k+1}}$ on the path $\pi$ are not constant in $v_k$. This is in contradiction with the assumption of the proposition that we have SPD along all directed paths. Therefore, we must have $\text{Ctf-IE} = 0$.
\end{proof}

In order to give a definition of PP along paths in terms of derivatives,  we want to compare derivatives between the prediction algorithm and the Data-Generating Process (DGP). To formalize the comparison, we consider two SCMs simultaneously: one describing the DGP, $\scmdef$, and one describing the prediction algorithm, $\scmhatdef$, including the prediction $\yhat$ and the features used to generate it. We assume that the variables in $\widetilde{\mathbf{V}}$ correspond to the variables in $\mathbf{V}$, that is, for each variable $V_k$ in $\mathbf{V}$ there exists a $\widetilde{V}_k$ in $\widetilde{\mathbf{V}}$. A variable $\widetilde{V}_k$ can be equivalent to the corresponding variable $V_k$, $\widetilde{V}_k \equiv V_k$, or generated as a part of a sequential prediction. If the variable is generated, the structural equations differ such that $f_{\widetilde{V}_k} \not= f_{V_k}$ and possibly the parents differ as well, $\parents_{\widetilde{V}_k} \not\equiv \parents_{V_k}$. In the previously mentioned example of satisfying SPD along the path $A \to P \to Y$, in the unfair process in Figure \ref{fig:problem}, $P$ in $\scm$ was generated as $\widehat{P}$ in $\scmhat$ and part of the sequential prediction $\yhat = f_{\yhat}(\widehat{P})$.  The structural equations of $P$ and $\widehat{P}$ differ and so do their parents, with $A$ in $\parents_P$ but not in $\parents_{\widehat{P}}$. We assume that at least the outcome, $Y$, is generated as $\yhat$ in $\scmhat$ and denote all variables in $\scmhat$ with $\widetilde{\cdot}$, regardless if they are generated or not.  We now define PP in derivatives along paths in the two SCMs.

\begin{definition}[Predictive Parity in Derivatives, PPD]
    \label{def:derivative-pp}
    Let $\mathbf{V}$ be a set of variables, including outcome $Y$, and $\widetilde{\mathbf{V}}$ be a set of partially generated variables corresponding to $\mathbf{V}$, including $\yhat$. For given SCMs $\scmdef$ and $\scmhatdef$ and a path $\pi : V_1 \to \dots \to V_k \to \dots \to Y$, we have that $\yhat$ satisfies path-specific predictive parity in derivatives along $\pi$ if, for the variables in $\pi$ and their counterparts in $\scmhat$, 
    \[
        \pder{}{\widetilde{v}_{k-1}}
        \Exp{\widetilde{V}_k \mid \widetilde{\parents}_{V_k} = \mathbf{p}}
        = 
        \pder{}{v_{k-1}}
        \Exp{ V_{k} \mid \parents_{V_k} = \mathbf{p}}
        \quad \forall k \in \{2, \dots, |\pi|\} \text{ and } 
        \forall\mathbf{p},
    \]
    where $\widetilde{\parents}_{V_k}$ are the variables in $\scmhat$ that correspond to $\parents_{V_k}$ in $\scm$ and  $\text{E}[\widetilde{V}_k \mid \widetilde{\parents}_{V_k} = \mathbf{p}]$ and $\text{E}[V_k \mid \parents_{V_k} = \mathbf{p}]$ are assumed to be differentiable.  
\end{definition}

\begin{remark}
\label{rem:ppd}
    If $\widetilde{V}_k$ is generated in a deterministic way, that is, $\widetilde{U}_k = \emptyset$, 
    and $\parents_{\widetilde{V}_k} \subseteq \widetilde{\parents}_{V_k}$, that is, no variables additional to the corresponding parents of $V_k$ are used in generating $\widetilde{V}_k$,
    then 
    $
    \text{E}[\widetilde{V}_k \mid \widetilde{\parents}_{V_k} = \mathbf{p}] 
    =
    f_{\widetilde{V}_k}(\mathbf{p})
    $.
    Moreover, if only a part of the path is modeled, then only the modeled part needs to be considered. In particular, if only direct effects on Y are modeled, that is, only $\yhat$ is generated and hence $\widetilde{\mathbf{V}}\setminus \yhat \equiv \mathbf{V} \setminus Y$, 
    then $\widetilde{\parents}_{Y} \equiv \parents_Y$ and predictive parity reduces to
    \[
        \pder{}{v_{|\pi| - 1}}
        f_{\yhat}(\mathbf{p})
        = 
        \pder{}{v_{|\pi| - 1}}
        \Exp{ Y \mid \parents_{Y} = \mathbf{p}}.
    \]
\end{remark}

The next result relates PPD to Causal PP (Definition \ref{def:causal-PP}) when the prediction algorithm only models direct effects to $Y$. 

\begin{proposition}
    \label{prop:pp-same}
    Let $\scmdef$ and $\scmhatdef$ be given SCMs where only $\yhat$ is generated (deterministically) in $\widetilde{\mathbf{V}},$ and $\parents_{\yhat} \equiv \parents_{Y}.$ Assume that $U_i \indep U_j \ \forall i, j \in \mathbf{U}$ and  consider a directed path $\pi: V_1 = X \to ... \to W \to Y.$ Then, if $\yhat$ satisfies PPD along $\pi$, $\yhat$ also satisfies causal predictive parity for all counterfactual contrasts in $X$ along $\pi$ in setting $\mathbf{Z} = \parents_Y$, according to Definition \ref{def:causal-PP}. 
\end{proposition}

\begin{proof}
    Since only $\yhat$ is generated,  we only need to consider the subpath $\pi^*: W \to Y$ according to Remark \ref{rem:ppd}. Assume that $\yhat$ satisfies PPD along $W \to Y$, then we have
    \[
        \pder{}{w}
        \Exp{ \yhat \mid \parents_{Y} = \mathbf{p}}
        = 
        \pder{}{w}
        \Exp{ Y \mid \parents_{Y} = \mathbf{p}},
    \]
    since $\parents_{\yhat} \equiv \widetilde{\parents}_Y \equiv \parents_Y$. Integrating both sides with respect to $w$ gives us the following equality
    \[
        \Exp{ \yhat \mid \parents_{Y} = \mathbf{p}}
        +
        c(\mathbf{p} \setminus w)
        = 
        \Exp{ Y \mid \parents_{Y} = \mathbf{p}}.
    \]
    where $c(\mathbf{p} \setminus w)$ is constant with respect to $w$. Since $U_i \indep U_j\ \forall i, j \in \mathbf{U}$, we have $Y(\pi^*(w)) \indep W \mid \parents_Y \setminus W$, and  
    \[
        \Exp{ Y \mid \parents_{Y} = \mathbf{p} }
        =
        \Exp{ Y \mid \parents_{Y} \setminus w = \mathbf{p} \setminus w, W=w }
        =
        \Exp{ Y(\pi^*(w)) \mid \parents_{Y} = \mathbf{p} }.
    \]
    Further, since $\yhat$ is deterministically generated, $\widetilde{U}_{\yhat} = \emptyset$, and since $\yhat$ is the only generated variable in $\widetilde{\mathbf{V}}$, $\mathbf{U} \setminus U_Y = \widetilde{\mathbf{U}}$, it follows that the same applies to $\yhat$. Therefore,
    \[
        \Exp{ \yhat(\pi^*(w)) \mid \parents_{Y} = \mathbf{p}}
        +
        c(\mathbf{p} \setminus w)
        = 
        \Exp{ Y(\pi^*(w)) \mid \parents_{Y} = \mathbf{p}} ,
    \]
   and for a contrast in $W$ along $\pi^*$ we have

    \begin{equation*}
        \begin{split}
        &\Exp{ \yhat(\pi^*(w_1)) \mid \parents_{Y} = \mathbf{p}}
        -
        \Exp{ \yhat(\pi^*(w_0)) \mid \parents_{Y} = \mathbf{p}} =\\
        &
        \Exp{ Y(\pi^*(w_1)) \mid \parents_{Y} = \mathbf{p} } -
        \Exp{ Y(\pi^*(w_0)) \mid \parents_{Y} = \mathbf{p} }.
        \end{split}
    \end{equation*}

    This corresponds to Definition \ref{def:causal-PP} when $\mathbf{Z} = \parents_Y$. 
\end{proof}

The assumption $U_i \indep U_j\ \forall i, j \in \mathbf{U}$ is an assumption of no unmeasured confounding and ensures, for example, that the direct effect of $W$ on $Y$ is identified and can thus be estimated using the observed data \citep{avin_identifiability_2005}. This assumption is not necessary for a predictor to satisfy our definition of PPD. 

In this section, we have defined SP and PP in terms of derivatives tailored to continuous features. However, when features are discrete, we suggest relying on contrasts, as in \cite{plecko_reconciling_2024} instead of derivatives. 

\subsection{Fair Predictor}
\label{sec:fairpred}

With the concepts of SPD and PPD defined above, we can describe a fair predictor in terms of not-allowed and allowed paths in an SCM.

\begin{definition}[Fair predictor]
    \label{def:fair-predictor}
    Let  $\pathset_{N}$ and $\pathset_{A}$ be sets of not-allowed and allowed paths ending in $Y$ respectively. Then, a prediction algorithm $\yhat$ is called a fair predictor with respect to $Y$ if
    \begin{enumerate}[label=\roman*.]
        \item $\yhat$ satisfies SPD with respect to all paths in $\pathset_\text{N}$, and
        \item $\yhat$ satisfies PPD with respect to all paths in $\pathset_\text{A}$.
    \end{enumerate}
\end{definition}
Note that a corresponding definition of a fair predictor in terms of Definition \ref{def:causal-SP} (for Ctf-DE, Ctf-IE) and  Definition \ref{def:causal-PP} is obtained by $\yhat$ satisfying Causal SP and Causal PP in all contrasts along $\pathset_\text{N}$ and $\pathset_\text{A}$, respectively. 

The existence of a fair predictor as defined in Definition \ref{def:fair-predictor} depends on the DGP, as the following result shows for the case where a not-allowed path represents a direct effect. 

\begin{theorem}
    \label{thm:compatibility}
     Let $\scmdef$ be an SCM, $\pi_\text{N}: X \to Y$ be a not-allowed path, and $\pi_\text{A}: W \to Y$ be an allowed path, where $X, W, Y \in \mathbf{V}$. Then, there exists a predictor $\yhat$ that satisfies SPD with respect to $\pi_\text{N}$ and PPD with respect to $\pi_\text{A}$ \textit{iff}
    \begin{equation}
    \label{eq:compatibility}
        \frac{\partial^2}{\partial w \partial x} 
        \Exp{Y \mid \parents_{Y} = \mathbf{p}}
        = 0.
   \end{equation}
\end{theorem}

\begin{proof}
    Assume that $\yhat$ is a deterministic predictor with $\parents_{\yhat} \subseteq \widetilde{\parents}_Y$, as in Remark \ref{rem:ppd},  
    and that $\yhat$ satisfies SPD and PPD with respect to $\pi_\text{N}: X \to Y$ and $\pi_\text{A}: W \to Y$, respectively. Then, from SPD we have that $\partial f_{\yhat} / \partial x = 0$,  and therefore $\partial^2 f_{\yhat} / \partial x \partial w = 0$. Further, from PPD we have that 
    \[
        \pder{}{w} 
        f_{\yhat}(\mathbf{p})
        = 
        \pder{}{w}
        \Exp{ Y \mid \parents_{Y} = \mathbf{p}},
    \]
    and hence we must have
    \[
        0
        = 
        \frac{\partial^2}{\partial w \partial x}
        f_{\yhat}(\mathbf{p})
        = 
        \frac{\partial^2}{\partial w \partial x}
        \Exp{ Y \mid \parents_{Y} = \mathbf{p}},
    \]
   where the first equality follows from SPD, while the second follows from PPD.
   
    Conversely, if
    \[
        \frac{\partial^2}{\partial w \partial x}
        \Exp{ Y \mid \parents_{Y} = \mathbf{p}}
        = 0
    \]
    we can write
    \[
        \Exp{ Y \mid \parents_{Y} = \mathbf{p}}
        = 
        g(w, \mathbf{z}) + h(x, \mathbf{z}).
    \]
    where $\mathbf{Z} = \parents_{Y} \setminus \{X, W\}$.
    Then, we can choose the predictor $\yhat = f_{\yhat}(w, \mathbf{z}) = g(w, \mathbf{z}) + h(c, \mathbf{z})$, which satisfies SPD along $\pi_\text{N}$ since
    \begin{equation*}
        \frac{\partial}{\partial x}
        f_{\yhat}(w, \mathbf{z}) 
        = 
        0
    \end{equation*}
    and PPD along $\pi_\text{A}$ since
    \begin{equation*}
        \frac{\partial}{\partial w}
        f_{\yhat}(w, \mathbf{z}) 
        = 
        \frac{\partial}{\partial w}
        g(w, \mathbf{z})
        =
        \frac{\partial}{\partial w}
        \Exp{ Y \mid \parents_{Y} = \mathbf{p}}.
    \end{equation*}
\end{proof}
This result tells us that SPD and PPD are not always compatible, and that Equation~\eqref{eq:compatibility} must hold for a fair predictor to exist. In Appendix \ref{app:compatibility}, we give an outline of a proof of a more general result, where indirect not-allowed paths are considered. However, in the sequel we focus on predictors modeling only direct effects for which Theorem \ref{thm:compatibility} is relevant. In practice, Equation \eqref{eq:compatibility} is a strong assumption, and if it does not hold, a compromise between SPD and PPD needs to be made. 

\section{Method for Fair Prediction}
\label{sec:method}
We now propose a method to achieve a fair predictor $\yhat$, if it exists according to Theorem \ref{thm:compatibility}. We focus here on deterministic predictors modeling direct effects on $Y$, and on cases where the not-allowed paths $\pathset_\text{N}$ are of length one, that is, represent direct effects. We also assume that $\widetilde{\parents}_{\yhat} \subseteq \parents_Y$, such that Remark \ref{rem:ppd} applies. Note that, in the more general case where not-allowed paths of length greater than one are of interest, the proposed method can be applied sequentially along these paths to construct a fair predictor, as illustrated in Appendix \ref{app:indirect-effects}. If a fair predictor does not exist, the following method allows the user to tune predictors to navigate the trade-off between SPD and PPD. For that purpose, the following constrained optimization problem is posited:
\begin{equation}
    \label{eq:fair-tuning-optimisation}
    \begin{split}
        \argmin{\theta} \quad 
        &\mathcal{L}(Y, \yhat_{\theta}) \\
        \text{subject to:} \quad
        & \nabla_{\pathset_N} f_{\yhat_\theta} = \mathbf{0}, \\
        & \nabla_{\pathset_A} f_{\yhat_\theta}(\mathbf{p})= \nabla_{\pathset_A} \text{E}[Y \mid \parents_{Y} = \mathbf{p}], 
    \end{split}
\end{equation}
where $\theta$ are the parameters of $\yhat_\theta$, $\mathcal{L}(Y, \yhat_{\theta})$ is a prediction loss, $\mathbf{0}$ is a vector of zeros, and $\nabla_{\pathset} f_{\yhat_\theta}$ is a vector of partial derivatives of $f_{\yhat_\theta}$ with respect to the parents of $\yhat_\theta$ along the paths in $\pathset$; see Appendix \ref{app:nabla-path} for definition. The resulting predictor satisfies both SPD with respect to the not-allowed paths and PPD with respect to the allowed paths if Equation \eqref{eq:compatibility} holds in the DGP. Note that no assumptions about the structural equations in the DGP need to be made to implement Equation \eqref{eq:fair-tuning-optimisation}, and thus the method presented here only requires a causal diagram.

To implement the above constrained optimization, we propose a tuning scheme for neural networks that we call \textit{fair tuning}. Fair tuning consists of first fitting an unconstrained predictor $\yhatu = f_{\yhatu}(\parents_Y)$ that minimizes a prediction loss targeting $\text{E}[Y \mid \parents_{Y}]$ in Equation \eqref{eq:fair-tuning-optimisation}, for example, Mean Squared Error (MSE) for a regression problem. Then, a fair predictor is obtained using the fair tuning loss with respect to $\yhatu$:
\begin{equation}
    \label{eq:fair-tuning-loss}
    \begin{split}
        \mathcal{L}_\text{FT} \left(\yhat_{\theta}, \yhatu ; \ls, \lp \right) 
        = 
        &\quad\underbrace{\mathcal{L}(\yhat_{\theta}, \yhatu)}_{\text{Pred. loss}} \\
        &+ \underbrace{\ls    || \nabla_{\pathset_N} f_{\yhat_\theta} ||_1}_{\text{SPD loss}} \\
        &+ \underbrace{\lp  || \nabla_{\pathset_A} f_{\yhat_\theta} - \nabla_{\pathset_A} f_{\yhatu} ||_1}_{\text{PPD loss}},
    \end{split}
\end{equation}
where $|| \cdot ||_1$ is the Manhattan ($L^1$) norm and $\ls$ and $\lp$ are parameters that give weight to SPD and PPD, respectively, allowing for prioritizing SPD or PPD when Equation \eqref{eq:compatibility} does not hold. The fair tuning algorithm is presented in its entirety in Algorithm \ref{alg:fair-tuning}. In Appendix \ref{app:time-complexity}, we show that the time complexity of backpropagation with Equation \eqref{eq:fair-tuning-loss} is of the same order as regular backpropagation.

The causal diagram $\diagram$ in Algorithm \ref{alg:fair-tuning} should indicate fairness constraints through allowed and not-allowed paths. The not-allowed paths decide which direct effects are minimized, removing any possible effect transmitted along these paths. Additionally, $\diagram$ can represent causal assumptions about the allowed paths by representing them using directed arrows, or, if no causal assumption is made, by double-dashed arrows. 

\begin{algorithm}[t]
\caption{Fair tuning}
\label{alg:fair-tuning}
\begin{algorithmic}[1]
    \Require
    Dataset $\mathcal{D} = \{\mathbf{x}, \mathbf{z},\mathbf{y}\}$ and
    Causal diagram $\diagram$ over $\mathbf{V} = \{\mathbf{X}, \mathbf{Z}, Y\}$, indicating allowed paths $\pathset_\text{A}$ and not-allowed paths $\pathset_\text{N}$ from protected attributes $\mathbf{X}$ to outcome $Y$.
    \Ensure Fair tuned predictor $\yhatft$ with parameters $\hat{\theta}_\text{FT}$.

    \Statex
    
    \Statex \textbf{Step 1:} Fit unconstrained predictor $\yhatu = f_{\yhatu}(\parents_Y)$, where $\parents_Y \in \mathbf{V} \setminus Y$ are the parents of $Y$ in $\diagram$, such that
    \begin{equation*}
        \hat{\theta}_\text{U} = \underset{\theta}{\text{argmin}} \quad \mathcal{L}(\mathbf{y}, \mathbf{\yyhat_\theta}(\pparents_Y))
    \end{equation*}
    where $\mathcal{L}$ is a prediction loss targeting $\text{E}[Y \mid \parents_Y]$.

    \Statex
    
    \Statex \textbf{Step 2:} Initialize fair predictor $\widehat{Y}_{\widetilde{\theta}_\text{FT}} = f_{\widehat{Y}_{\widetilde{\theta}_\text{FT}}}(\parents_Y)$ using the parameter values $\hat{\theta}_\text{U}$ of $\yhatu$,
    \begin{equation*}
        \widetilde{\theta}_\text{FT} \gets \hat{\theta}_\text{U}.
    \end{equation*}

    \Statex
    
    \Statex \textbf{Step 3:} Define $\mathcal{L}_\text{FT}$ (Equation \ref{eq:fair-tuning-loss}) with respect to $\pathset_\text{A}$ and $\pathset_\text{N}$.

    \Statex
    
    \Statex \textbf{Step 4:} Tune $\widehat{Y}_{\widetilde{\theta}_\text{FT}}$, using the defined $\mathcal{L}_\text{FT}$, to obtain parameters $\hat{\theta}_\text{FT}$ such that 
    \begin{equation*}
        \begin{split}
            \hat{\theta}_\text{FT} =
            \underset{\theta}{\text{argmin}} \quad
            \mathcal{L}_\text{FT} \left(\mathbf{\yyhat_\theta}, \mathbf{\yyhat_{\theta_\text{U}}} ; \ls, \lp \right).
        \end{split}
    \end{equation*}

    \Statex
    
    \Statex \Return{$\yhatft$}.
\end{algorithmic}
\end{algorithm}

When implementing Algorithm \ref{alg:fair-tuning}, there are some practical considerations to be made. In Step 1, when fitting the unconstrained predictor, overfitting should be avoided by using strategies such as tuning the training hyperparameters using cross-validation or cross-fitting. Further, it is important to use a twice-differentiable activation function, such as the Exponential Linear Unit (ELU), since both the first derivative with respect to the features and the second derivative with respect to the network parameters are required in the computations of the SPD and PPD losses in Equation \eqref{eq:fair-tuning-loss}.  In Step 2, when initializing the fair predictor, $\widehat{Y}_{\widetilde{\theta}_\text{FT}}$, using the parameters of the unconstrained predictor $\hat{\theta}_\text{U}$ ensures that both prediction and PPD loss in Equation \eqref{eq:fair-tuning-loss} are minimized at the start of tuning. Finally, when tuning in Step 4, the prediction loss is computed with respect to $\yhatu$ to avoid overfitting to the data. Therefore, all available data can be used during the tuning step.

\section{Numerical Experiments}
\label{sec:numerical_experiments}
We conduct numerical experiments to illustrate the cost in predictive performance of fair tuning, while varying SPD and PPD tuning and signal-to-noise ratios.\footnote{Code with examples and to replicate the simulation study is available at \url{https://github.com/fileds/tuning-derivatives-for-causal-fairness-in-machine-learning}.} Note that loss of predictive performance is expected and even desirable when the aim is to avoid bias in the DGP. We also compare the proposed fair tuning strategy (labeled FT) with two other fairness strategies, which work with continuous-valued protected attributes: one focusing only on SP tuning (labeled SPT), based on \cite{kancheti_matching_2022}, and the other one based on the marginalization scheme of \cite{lindholm_discrimination-free_2022} (labeled Marginalize), described below. All strategies are based on an unconstrained predictor (labeled Unconstrained). Finally, we study our compatibility result (Theorem \ref{thm:compatibility}) under two settings: one linear, where SPD and PPD are compatible, and one non-linear, where they are not. Thus, we consider the following DGPs:
\begin{align}
    \label{eq:numerical-scm}
    U_{X, Z} &\sim \text{N}(0, 1) \nonumber \\
    X &= U_{X, Z} + U_X \nonumber \\
    Z &= U_{X, Z} + U_Z \nonumber \\
    W &= X + U_W \nonumber \\
    Y_\text{Linear} &= - X + Z + W + U_Y \nonumber \\
    Y_\text{Multiplicative} &= X \cdot Z \cdot W + U_Y,
\end{align}
where $X$ and $Z$ are correlated through the unobserved $U_{X, Z}$. $U_X, U_Z, U_W$ are sampled independently from $\text{N}(0, 1)$ and $U_Y$ from $\text{N}(0, \sigma^2)$. The DGPs are illustrated with the diagram in Figure \ref{fig:numerical-diagram}, where we also display the not-allowed (red, squiggly) and allowed paths (green, straight), that is, we consider the direct effect of $X$ to be not allowed, but all other paths to be allowed. 

\begin{figure}[ht]
    \centering
    \begin{tikzpicture}
        \node[neutral node] (X) at (-0.8, 0) {};
        \node[neutral node] (Z) at (0, 0.8) {};
        \node[neutral node] (W) at (0, -0.8) {};
        \node[neutral node] (Y) at (0.8, 0) {};
        
        \node[left] at (X.west) {$X$};
        \node[above] at (Z.north) {$Z$};
        \node[below] at (W.south) {$W$};
        \node[right] at (Y.east) {$Y$};
        
        \draw[<->, fair edge, dashed] (X) to[bend left] (Z);
        \draw[->, fair edge] (X) to (W);
        \draw[->, unfair edge] (X) to (Y);
        \draw[->, fair edge] (Z) to (Y);
        \draw[->, fair edge] (W) to (Y);
    \end{tikzpicture}
    \caption{Diagram describing the DGP described by Equation \eqref{eq:numerical-scm} used to generate data for numerical experiments. $X$ and $Z$ are correlated (dashed, double-headed arrow) through an unobserved background variable. The not-allowed path is $X \to Y$ (red, squiggly arrow). All other paths are allowed (green arrows).}
    \label{fig:numerical-diagram}
\end{figure}
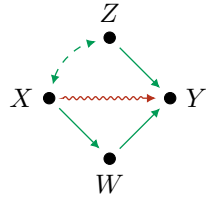

\subsection{Predictors}
We implement fully connected, feed-forward neural networks with two hidden layers, using 32 nodes in the Linear setting and 64 nodes in the Multiplicative setting. The network sizes are chosen to reflect the complexity of the DGP and to provide sufficient flexibility for subsequent tuning. We apply the ELU activation between layers and train the unconstrained predictor $\yhatu$ using stochastic gradient descent using the ADAM optimizer \citep{kingma_adam_2017} with a batch size of 64. We train for 50 epochs in the Linear setting and for 200 epochs in the Multiplicative setting. 

We implement SPT ($\yhatspt$) and FT ($\yhatft$) by tuning $\yhatu$ with the loss function described in Equation \eqref{eq:fair-tuning-loss}. For SPT we set $\lp = 0$, that is, no PPD tuning, and for FT we set $\ls = \lp$. Note that SPT is analogous to FTU, that is, excluding the protected attribute when minimizing the SPD loss. The Marginalize predictor is implemented by replacing the value of $X$ by the mean value of $X$, that is, $\yhat_{\text{Marginalize}} = f_{\yhatu}(\Empexp{X}, W, Z)$, where $\Empexp{X}$ is the empirical expectation over the training data. We tune for 20 epochs in the Linear setting and 100 epochs in the Multiplicative setting. All predictors are implemented in PyTorch \citep{ansel_pytorch_2024} and gradients are computed using the autograd sub-library.

\subsection{Design}
In both settings, we simulate 500 replicates for six different signal-to-noise scenarios by setting $\sigma^2$ to $0, 1, 4, 9, 16, 25$. For each replicate, we generate a training and testing dataset with 1000 samples in each set. We vary the tuning parameters for $\yhatspt$ and $\yhatft$ in each scenario as follows: $\ls = \lp \in \{0.5 \sigma^2, \sigma^2, 10 \sigma^2, 100 \sigma^2\}$. Scaling by $\sigma^2$ allows us to maintain the proportions between the prediction loss (MSE) and SPD and PPD losses in Equation \eqref{eq:fair-tuning-loss}, as the oracle prediction loss in each scenario is given by $\sigma^2$. We evaluate the predictors on the test set using MSE of prediction, $\text{E}_n[Y - \yhat_\theta]^2$, the magnitude of the gradients over the not-allowed paths, $|| \nabla_{\pathset_N} f_{\yhat_\theta} ||_1$, and the Mean Absolute Error (MAE) of the gradients with respect to the true gradients in the DGP over the allowed paths, $|| \nabla_{\pathset_A} f_Y - \nabla_{\pathset_A} f_{\yhat_\theta} ||_1$. The true gradients in the DGPs along the allowed paths are 1 in the Linear setting and $x \cdot z$ for $X \to W$ and $x \cdot w$ for $Z \to Y$ in the Multiplicative setting, as defined in Equation \eqref{eq:numerical-scm}. The evaluation corresponds to the Prediction, SPD, and PPD losses in Equation \eqref{eq:fair-tuning-loss}, but with respect to DGP instead of $\yhatu$. 

\subsection{Results}

The results for $\ls = \lp \in \{\sigma^2, 100 \sigma^2\}$ are visualized in Figure \ref{fig:numerical-results}, where the prediction loss is shown in the first column, the SPD loss in the second column, and PPD losses with respect to $W \to Y, Z \to Y$ in columns three and four respectively. The SPT and FT predictors are indexed by their scaling of $\ls$ and $\ls = \lp$, respectively. Results for all levels of $\ls$ and $\lp$ can be found in Figure \ref{fig:all-results} in Appendix \ref{app:all-results}. 

\begin{figure}[t]
    \centering
    \begin{subfigure}[b]{\textwidth}
        \centering
        \includegraphics{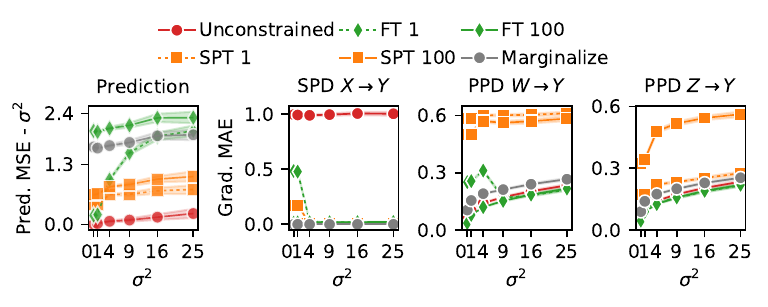}
        \caption{Linear setting}
        \label{fig:results-linear}
    \end{subfigure}
    \hfill
    \begin{subfigure}[b]{\textwidth}
        \centering
        \includegraphics{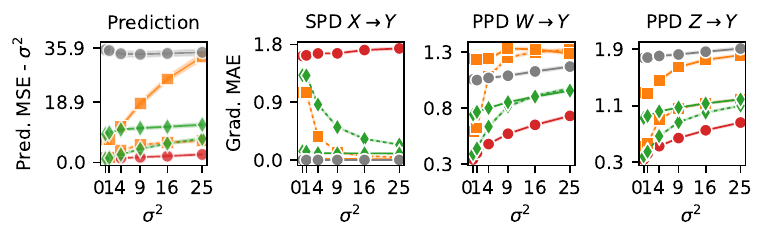}
        \caption{Multiplicative setting}
        \label{fig:results-multiplicative}
    \end{subfigure}
    \caption{Prediction loss (shifted by oracle MSE $\sigma^2$), SPD loss, and PPD loss for varying $\sigma^2$ values in (a) the Linear and (b) the Multiplicative setting. Lower is better for all losses. Results are averaged over 500 replicates with 95\% confidence intervals. SPT and FT predictors are indexed by their $\ls$ and $\ls=\lp$ tuning, respectively.}
    \label{fig:numerical-results}
\end{figure}

Let us first discuss predictive performance, the first column in Figure \ref{fig:numerical-results}. As expected and intended, predictive performance with respect to the unfair outcome $Y$ decreases when enforcing fairness constraints, and the Unconstrained predictor achieves the lowest MSE in all scenarios. The stronger the constraints, the worse the predictive performance. This is illustrated by the decreasing performance of SPT and FT predictors as $\ls$ and $\lp$ increase and by the Marginalize predictor, which enforces SPD by design, and performs poorly in both settings compared to the Unconstrained predictor. The SPT predictors show better predictive performance than the FT predictors in the Linear setting, but the relationship is reversed in the Multiplicative setting. In the Linear setting, the better performance of the SPT predictors comes at a cost in PPD, see columns three and four in Figure \ref{fig:numerical-results}. For the SPT 100 predictor, the sum of the increase in PPD for $W \to Y$ and for $Z \to Y$ is similar to the difference between the Unconstrained predictor and the SPT predictor in SPD, column two in  Figure \ref{fig:numerical-results}. This is further illustrated by the mean gradients of the SPT 100 predictor, visualized in Figure \ref{fig:results-linear-mean-gradient} in Appendix \ref{app:mean-gradient}, where the mean of the gradient $\partial \yhat_{\text{SPT 100}} / \partial X$ increases by 1, and the mean of the gradients $\partial \yhat_{\text{SPT 100}} / \partial W, \partial \yhat_{\text{SPT 100}} / \partial Z$ decrease by about 0.5 each. The SPT predictors thus compensate for the loss in prediction power due to the SPD constraint by attributing variance in $Y$ due to $X$ to the variables that are correlated with $X$, that is, $W, Z$. Conversely, the improved predictive performance of the FT predictors in the Multiplicative setting comes at a cost in SPD. However, this cost is not proportional to the improvement in PPD, as in the SPT case. In the Multiplicative setting, when SPD and PPD are not compatible, strong SPD constraints result in poor predictive performance, as shown by the SPT 100 and Marginalize predictors. For all predictors, the predictive performance decreases as the noise variance $\sigma^2$ increases. 

We now shift focus to the fairness properties displayed in the second, third, and fourth columns in Figure \ref{fig:numerical-results}. Marginalize achieves zero SPD loss by definition and, as expected from the fair tuning loss (Equation \ref{eq:fair-tuning-loss}), the SPT predictors achieve lower SPD loss than their FT counterparts, while the FT predictors perform better in PPD. The interesting difference occurs when comparing the results between the two settings. From Theorem \ref{thm:compatibility}, we expect SPD and PPD to be compatible in the Linear setting but not in the Multiplicative setting. Thus, we expect the FT predictors to achieve lower SPD and PPD loss as $\ls = \lp$ increases in the Linear setting. However, in the Multiplicative setting, we expect that when one loss decreases, the other increases. We observe that this is indeed the case in Figure \ref{fig:numerical-results}, as FT 100 performs better in both SPD and PPD compared to FT 1 for the Linear setting, but in the Multiplicative setting FT 100 performs better in SPD, but not so in PPD. This is further illustrated for several values of $\ls = \lp$ in Figure \ref{fig:ft-only} in Appendix \ref{app:ft-results}. Finally, we expect a strong constraint on SPD to result in poor performance in both prediction and PPD in the Multiplicative setting, but not so in the Linear setting. Marginalize does indeed perform reasonably well in prediction and very well in SPD and PPD in the Linear setting, but the predictive performance in the Multiplicative setting is poor, and its performance in PPD is worse compared to the other fairness methods. SPT 100 also exhibits poor predictive performance in the Multiplicative setting as a result of the high value of $\ls$. These results are further illustrated with Pareto fronts \citep{deb_multi-objective_2011} in Figure \ref{fig:pareto_numerical} in Appendix \ref{app:pareto}.

\section{Real-World Case Study: COMPAS}
\label{sec:real_world_example}

We apply the introduced fair tuning method to the COMPAS dataset to create a prediction algorithm that predicts recidivism while considering several protected attributes. We curate a dataset from the original data by selecting the following features: sex, age, race, number of prior offenses, charge degree, and the binary outcome recidivism after two years.\footnote{The data used in the original analysis can be found here \url{https://github.com/propublica/compas-analysis/} (accessed 18/06/2025).} We restrict our analysis to include only African-American (60\%) and Caucasian (40\%) people. Sex is binary in the dataset, male (80\%) and female (20\%). Age is used as a continuous feature, and the recidivism rate is 47\% in the final dataset. We adopt the diagram used in \cite{plecko_reconciling_2024}, see  Figure \ref{fig:compas-diagram}, and consider all direct paths from the protected attributes, race, sex, and age, to the outcome to be not allowed, but any path through priors and charge degree to be allowed. This allows us to show how fair tuning may handle multiple protected attributes simultaneously.

\begin{figure}[ht]
    \centering
    \begin{tikzpicture}
        \node[neutral node] (X) at (-1.6, 0) {};
        \node[neutral node] (Z1) at (-0.8, 0.8) {};
        \node[neutral node] (Z2) at (0.8, 0.8) {};
        \node[neutral node] (W1) at (-0.8, -0.8) {};
        \node[neutral node] (W2) at (0.8, -0.8) {};
        \node[neutral node] (Y) at (1.6, 0) {};
        
        \node[left] at (X.west) {Race};
        \node[above] at (Z1.north) {Sex};
        \node[above] at (Z2.north) {Age};
        \node[below] at (W1.south) {Priors};
        \node[below] at (W2.south) {Degree};
        \node[right] at (Y.east) {Recidivism};
        
        \draw[<->, fair edge, dashed] (X) to[bend left=20] (Z1);
        \draw[<->, fair edge, dashed] (X) to[bend left=20] (Z2);
        \draw[<->, fair edge, dashed] (Z1) to[bend left=20] (Z2);
        \draw[->, fair edge] (X) to (W1);
        \draw[->, fair edge] (X) to (W2);
        \draw[->, unfair edge] (X) to (Y);
        \draw[->, unfair edge] (Z1) to (Y);
        \draw[->, unfair edge] (Z2) to (Y);
        \draw[->, fair edge] (Z1) to (W1);
        \draw[->, fair edge] (Z1) to (W2);
        \draw[->, fair edge] (Z2) to (W1);
        \draw[->, fair edge] (Z2) to (W2);
        \draw[->, fair edge] (W1) to (W2);
        \draw[->, fair edge] (W1) to (Y);
        \draw[->, fair edge] (W2) to (Y);
    \end{tikzpicture}
    \caption{Diagram describing the assumed relationships between the variables in the COMPAS dataset where recidivism is the outcome. Not-allowed paths are illustrated by red, squiggly arrows, and allowed paths by green arrows. Induced dependence is marked with double-headed, dashed arrows.}
    \label{fig:compas-diagram}
\end{figure}
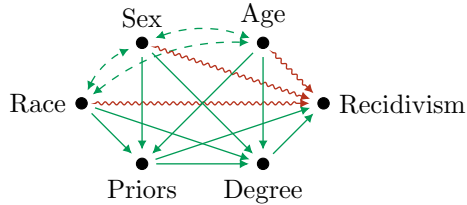

\subsection{Predictors}
We implement feed-forward, fully connected neural networks with two hidden layers with 64 nodes in each layer using the ELU activation function, and apply a sigmoid activation on the logits $\widehat{Z}$, that is, the output of the final linear layer of the network, before binarizing using a threshold of 0.5. Since binarization is not differentiable, and we are interested in the derivatives of the predictors, we make use of the logits both during the fitting of the unconstrained predictor and the tuning process.

When fitting the unconstrained predictor, we use binary cross-entropy loss, which targets $\text{E}[Y \mid \parents_Y],$ where $Y$ is the binary outcome representing recidivism. To avoid overfitting,  we use a cross-fitting strategy. We split the training data into five folds. For each fold, we train a neural network on the remaining folds for 100 epochs, and use the fitted neural network to generate out-of-fold logits, $\widehat{Z}_\text{OOF}$, for the current, held-out fold. We then train the unconstrained predictor for 100 epochs on the full dataset to minimize MSE with respect to $\widehat{Z}_\text{OOF}$, that is,
\begin{equation*}
    \hat{\theta}_\text{U} = \underset{\theta}{\text{argmin}} \quad 
    \frac{1}{N} \sum_{i=1}^N (\widehat{Z}_\text{OOF} - \widehat{Z}_\theta)^2,
\end{equation*}
where N is the number of samples.

We consider the same predictors with fairness constraints as in the previous section, that is, SPT, FT, and Marginalize, with $\ls = 10$ and $\lp = 0, 1, 10$ resulting in SPT 10, FT 1, and FT 10. We apply fair tuning for 50 epochs, minimizing the MSE between the logits of the target predictor and the unconstrained predictor. For example, to obtain FT 10 we apply Equation \eqref{eq:fair-tuning-loss} as
\begin{equation*}
    \hat{\theta}_{\text{FT }10} = \underset{\theta}{\text{argmin}} \quad \mathcal{L}_{\text{FT}}(\widehat{Z}_\theta, \widehat{Z}_{\hat{\theta}_\text{U}}; 10, 10),
\end{equation*}
where prediction loss is given by the MSE. When evaluating the fair tuning loss for binary features, we take the pointwise derivatives at the values 0 and 1, applying Equation \eqref{eq:fair-tuning-loss} as is. This works in practice since neural networks are continuous functions of their input. The alternative is to use contrasts, as in Definitions \ref{def:causal-SP} and \ref{def:causal-PP}, instead, see the discussion in Section \ref{sec:compas-comparison}. The Marginalize predictor uses the training set mode of the binary protected attributes instead of the mean.

\subsection{Results}
We apply the fair tuning method with the above-mentioned cross-fitting for 1000 bootstraps over the original dataset to estimate the sampling distribution of predictive and fairness performance measures.  Because the outcome is binary, predictive performance is measured by accuracy, F1 score, and Area Under the Receiver Operating Curve (AUC-ROC). These, together with fairness performance in SPD and PPD loss, are displayed in Figure \ref{fig:compas-performance} with the corresponding numbers in Table \ref{tab:compas-results} in Appendix \ref{app:compas-tables}.
\begin{figure}[t]
    \centering \includegraphics{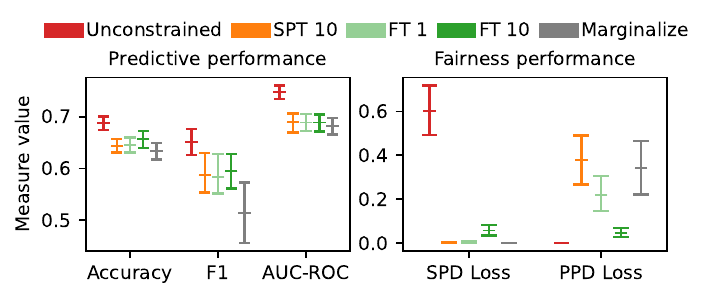}
    \caption{Predictive (Accuracy, F1, AUC-ROC) and fairness (SPD, PPD Loss) performance on the COMPAS dataset. Mean values and 95\% confidence bands are computed over 1000 bootstrap samples of the original dataset. Accuracy and F1 use a 0.5 threshold; AUC-ROC is the area under the ROC curve; PPD Loss is relative to the Unconstrained predictor. Higher is better for Accuracy, F1, AUC-RUC and lower is better for SPD and PPD Loss.}
    \label{fig:compas-performance}
\end{figure}

As expected, fairness constraints yield reduced predictive performance. 
However, the predictive performance is similar across all predictors with constraints, except for Marginalize, which performs poorly with respect to the F1 score. The most notable difference is in terms of PPD loss, where FT 10 performs significantly better than all other predictors with constraints, at a cost of performing slightly worse in terms of SPD loss. As expected, FT 1 performs slightly better in terms of PPD loss when compared to SPT 10 and Marginalize, while performing similarly in terms of SPD Loss. FT 10 provides the best compromise between SPD and PPD, as seen in Figure \ref{fig:compas-performance} and in the Pareto front, shown in Figure \ref{fig:compas-pareto} in Appendix \ref{app:compas-pareto}.

\subsection{Comparison of Derivative and Contrast Fairness Notions}
\label{sec:compas-comparison}
In the previous section, we applied the fair tuning loss  (Equation \ref{eq:fair-tuning-loss}) as is to the binary features. An alternative would be to use contrast fairness notions Causal SP and PP (Definitions \ref{def:causal-SP} and \ref{def:causal-PP}) instead. We implement Definitions \ref{def:causal-SP} and \ref{def:causal-PP} as loss functions, analogous to SPD and PPD losses, to study the effect of tuning for derivative fairness on contrast fairness. For each observation $i$, we create the contrast feature vectors  $\mathbf{x}_i^0$ and $\mathbf{x}_i^1$, where the attribute to contrast  is set to 0 in $\mathbf{x}_i^0$ and 1 in $\mathbf{x}_i^1$, while the other attributes are unchanged. We then define the CSP and CPP losses as 
\begin{equation*}
    \begin{split}
        &\mathcal{L}_\text{CSP}(\widehat{Z}_\theta) = \frac{1}{N} \sum_{i=1}^N 
        \left|
        \hat{z}_\theta(\mathbf{x}_i^1) 
        - \hat{z}_\theta(\mathbf{x}_i^0)
        \right|, \\
        &\mathcal{L}_\text{CPP}(\widehat{Z}_\theta, \widehat{Z}_{\hat{\theta}_\text{U}}) = 
        \frac{1}{N} \sum_{i=1}^N 
        \left|
        \left( \hat{z}_{\theta}(\mathbf{x}_i^1) - \hat{z}_{\theta}(\mathbf{x}_i^0) \right)
        - 
        \left( \hat{z}_{\hat{\theta}_\text{U}}(\mathbf{x}_i^1) - \hat{z}_{\hat{\theta}_\text{U}}(\mathbf{x}_i^0) \right)
        \right|,
    \end{split}
\end{equation*}
where $\widehat{Z}_\theta$ and $\widehat{Z}_{\hat{\theta}_\text{U}}$ are the logits corresponding to predictors $\yhat_{\theta}$ and $\yhatu$, respectively. We compute CSP loss for race and sex and CPP loss for charge degree, and compare feature-wise with SPD and PPD losses, respectively. 

The results are shown in Figure \ref{fig:compas-comparison} with the corresponding numbers in Table \ref{tab:compas-comparison} in Appendix \ref{app:compas-tables}.
\begin{figure}[t]
    \centering
    \includegraphics{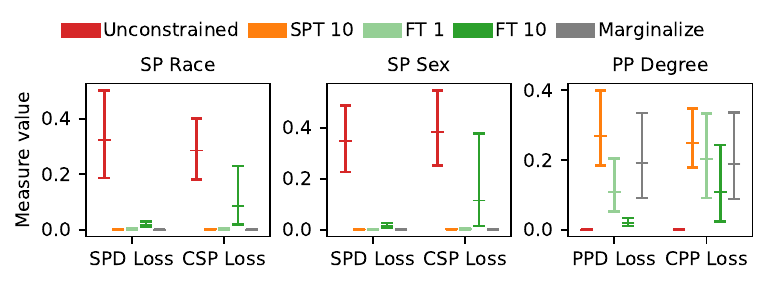}
    \caption{Comparison between derivative fairness (SPD and PPD) and contrast fairness (CSP and CPP) losses for binary features of the COMPAS dataset. Mean values and 95\% confidence bands are computed over 1000 bootstrap samples of the original dataset. Lower is better for all measures. PPD and CPP are computed w.r.t. the Unconstrained predictor.}
    \label{fig:compas-comparison}
\end{figure}
We see that minimizing SPD loss successfully minimizes CSP loss for SPT 10 and FT 1, even if the assumption of compact support in SPD (Definition \ref{def:derivative-sp}) is violated for binary features. The results are less encouraging for FT 10, which exhibits a discrepancy between SPD loss and CSP loss for both race and sex. Similarly, FT 10 shows the biggest discrepancy between PPD and CPP loss for charge degree. However, it is encouraging that the order of the tuned predictors (SPT 10, FT 1, FT 10) remains the same between derivative and contrast fairness. Ultimately, if Causal SP and PP is desired, CSP and CPP losses should be implemented in the tuning. 

\section{Conclusion}
\label{sec:conclusion}
In this paper, we formulate the notions of statistical and predictive parity in terms of partial derivatives along paths, SPD and PPD, to better handle continuous variables, building on previous results in path-specific fairness based on contrasts \citep{plecko_reconciling_2024}. We show the relation between our definitions and previous contrast definitions, and derive and study a compatibility criterion for when SPD and PPD are possible to achieve simultaneously. We define a fair predictor in terms of SPD and PPD, and develop a fair tuning method to construct a fair predictor when the compatibility criterion is met. When the compatibility criterion is not fulfilled, fair tuning offers a compromise between SPD and PPD. We compare our method with previous methods that handle continuous attributes and show that our method creates a fair predictor when the criterion is satisfied and offers a flexible way to compromise between statistical and predictive parity when the criterion is not met, better adhering to the concepts of statistical and predictive parity. Finally, we compare derivative fairness notions with contrast fairness notions \citep{plecko_reconciling_2024} when applied to binary protected attributes, and conclude that, in such cases, derivative fairness may result in contrast fairness in practice, but not generally.

\bmhead{Acknowledgements}
We are grateful to Mohammad Ghasempour and two anonymous referees for insightful comments that have helped us improve this manuscript. We also acknowledge funding from the Marianne and Marcus Wallenberg Foundation.

\appendix
\section*{Appendix}
\addcontentsline{toc}{section}{Appendix}

\section{Compatibility of Indirect Not-Allowed Paths}
\label{app:compatibility}
Version of Theorem \ref{thm:compatibility} that covers indirect paths. 

\begin{proposition}
    \label{prop:compatibility-indirect}
    Let $\scmdef$ be an SCM, $\pi_\text{N}: X \to ... \to Y$ be a not-allowed path, and $\pi_\text{A}$ be an allowed path (ending in $Y$). Then, there exists a predictor $\yhat$ that satisfies SPD with respect to $\pi_\text{N}$ and PPD with respect to $\pi_\text{A}$ if there exists at least one variable $V$ such that $V \in \pi_\text{N}$ and $V \neq X$ that satisfies 
    \[
        \frac{\partial^2}{\partial v_n \partial v_a} 
        \Exp{ V \mid \parents_{V} = \mathbf{p}} 
        = 0
    \]
    where $V_n, V_a$ are the parents of $V$ along $\pi_\text{N}$ and $\pi_\text{A}$ respectively or $V$ is not in $\pi_A$.
\end{proposition}

\textbf{Proof outline:} If $V$ is not in $\pi_A$, $V$ may be generated as $\widehat{V}$ such that SPD is satisfied along $\pi_\text{N}$ without affecting $\pi_\text{A}$. If $V$ is in $\pi_A$, the result follows from replacing $Y$ and $\yhat$ with $V$ and $\widehat{V}$, respectively, in the proof of Theorem \ref{thm:compatibility}. Any sequential predictor based on $\widehat{V}$, which satisfies SPD along the subpath $X \to ... \to V$, would then satisfy SPD along $\pi_\text{N}$ (follows from the definition of SPD) without affecting PPD. 

\section{Partial Derivatives with Respect to Paths}
\label{app:nabla-path}
For notational convenience, we define the operator $\nabla_{\pathset}$ with respect to a set of paths $\pathset = \{\pi_1, ..., \pi_n\}$ that end in $V$ in an SCM $\scmdef$ as the $n$-dimensional vector
\begin{equation*}
    \nabla_{\pathset} f_V = 
    \mathbf{e}_1 \pder{f_{V}}{\text{pa}_1} +
    ... +
    \mathbf{e}_n \pder{f_{V}}{\text{pa}_n},
\end{equation*}
where $f_{V}$ is the mechanism in $\mathbf{F}$ assigning values to $V$, $\text{pa}_k$ is the parent of $V$ along path $\pi_k$, $\mathbf{e}_i$ is the $i$th vector of the canonical basis of $\mathbb{R}^n$. 

For example, if we have a set of paths $\pathset = \{\pi_1, \pi_2\}$, where $\pi_1: X \to W \to Y, \pi_2: X \to Z \to Y$, then 
\begin{equation*}
    \nabla_{\pathset} f_Y= 
    \mathbf{e}_1 \pder{f_Y}{w} +
    \mathbf{e}_2 \pder{f_Y}{z}.
\end{equation*}

In the second constraint in Equation \eqref{eq:fair-tuning-optimisation}, when $\nabla_{\pathset}$ is applied to the expectation $\text{E}[Y \mid \parents_Y = \mathbf{p}]$, the expectation should be interpreted as a function of $\mathbf{p}$ with parents $\parents_Y$.

\section{Fair sequential prediction}
\label{app:indirect-effects}
We illustrate how a fair sequential predictor can be achieved to illustrate how to handle indirect not-allowed paths. For the sake of illustration, we consider a DGP with structural equations
\begin{equation}
\label{eq:indirect-effects-dgp}
    \begin{split}
        U_{X, Z} &\sim \text{N}(0, 1), \\
        X &= U_{X, Z} + U_X, \\
        Z &= U_{X, Z} + U_Z, \\
        W &= \beta_{X,W} X + \beta_{Z,W} Z + U_W, \\
        Y &= \beta_{X,Y} X + \beta_{W,Z,Y} W Z + U_Y,
    \end{split}
\end{equation}
where $U_X, U_Z, U_W$, and $U_Y$ are sampled from $\text{N}(0,1)$. Note that there $Z$ appears in the mechanism of $W$, and that there is an interaction term between $W$ and $Z$ in the mechanism of $Y$. We consider $X \to Y$ and $X \to W \to Y$ to be not-allowed paths from $X$ to $Y$, $\pathset_\text{N}$, and consider all remaining paths to be allowed, $\pathset_\text{A}$, see Figure \ref{fig:indirect-effects-dgp} for an illustration.

\begin{figure}[ht]
    \centering
    \begin{tikzpicture}
        \node[neutral node] (X) at (-0.8, 0) {};
        \node[neutral node] (Z) at (0, 0.8) {};
        \node[neutral node] (W) at (0, -0.8) {};
        \node[neutral node] (Y) at (0.8, 0) {};
        
        \node[left] at (X.west) {$X$};
        \node[above] at (Z.north) {$Z$};
        \node[below] at (W.south) {$W$};
        \node[right] at (Y.east) {$Y$};
        
        \draw[<->, fair edge, dashed] (X) to[bend left] (Z);
        \draw[->, unfair edge] (X) to (W);
        \draw[->, unfair edge] (X) to (Y);
        \draw[->, fair edge] (Z) to (W);
        \draw[->, fair edge] (Z) to (Y);
        \draw[->, unfair edge] (W) to (Y);
    \end{tikzpicture}
    \caption{Diagram describing the DGP described by Equation \eqref{eq:indirect-effects-dgp} used to generate data for numerical experiments. $X$ and $Z$ are correlated (dashed, double-headed arrow) through an unobserved background variable. The not-allowed paths are illustrated as red, squiggly arrows. All other paths are allowed.}
    \label{fig:indirect-effects-dgp}
\end{figure}
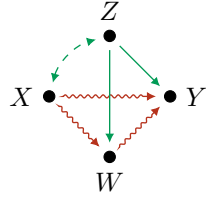

Due to the interaction term in the mechanism of $Y$, the condition in Theorem \ref{thm:compatibility} is not satisfied for $Y$ w.r.t. not-allowed path $X \to W \to Y$ and allowed path $Z \to Y$:
\begin{equation*}
    \frac{\partial}{\partial z \partial w} \Exp{Y \mid x, w, z}= \beta_{W,Z,Y} \neq 0.
\end{equation*}
Therefore, applying fair tuning directly to $Y$ cannot result in a fair predictor. Instead, we can rely on Proposition \ref{prop:compatibility-indirect} in Appendix \ref{app:compatibility}, which states that a fair predictor can be achieved if one variable along the not-allowed path satisfies the condition w.r.t. the allowed path. $W$ satisfies the condition w.r.t. not-allowed path $X \to W$ and allowed path $Z \to W$ as
\begin{equation*}
    \frac{\partial}{\partial z \partial x} \Exp{W \mid x, z} = 0.
\end{equation*}

Assuming that we know the structural equations of the DGP (Equation \ref{eq:indirect-effects-dgp}), we can generate $W$ as $\widehat{W} = f_{\widehat{W}}(Z) = \beta_{Z,W} Z $ which satisfies SPD along $X \to W$ and PPD along $Z \to W$. Using $\widehat{W}$, we can create the sequential predictor
\begin{equation*}
    \begin{split}
        \widehat{W} &= f_{\widehat{W}}(Z) = \beta_{Z,W} Z \\
        \yhat &= f_{\yhat}(\widehat{W},Z) = \beta_{W,Z,Y} \widehat{W} Z.
    \end{split}
\end{equation*}
Then $\yhat$ satisfies SPD along the not-allowed paths as
\begin{equation*}
    \frac{\partial f_{\yhat}}{\partial x} 
    =
    \frac{\partial f_{\yhat}}{\partial \widehat{w}} 
    \frac{\partial f_{\widehat{W}}}{\partial x} 
    = 0,
\end{equation*}
since $\frac{\partial f_{\widehat{W}}}{\partial x} = 0$, and
PPD along $Z \to W \to Y$ as
\begin{align*}
    &\frac{\partial }{\partial \widehat{w}} \Exp{\yhat \mid x, \widehat{w}, z}
    =
    \beta_{W,Z,Y} z 
    =
    \frac{\partial}{\partial w} \Exp{Y \mid x, w, z}
    \quad \text{and}\\[0.5em]
    &\frac{\partial }{\partial z} \Exp{\widehat{W} \mid x, z}
    =
    \beta_{Z,W}
    =
    \frac{\partial}{\partial z} \Exp{W \mid x, z},
\end{align*}
and PPD along $Z \to Y$ as
\begin{equation*}
    \frac{\partial }{\partial z} \Exp{\yhat \mid x, \widehat{w}, z}
    = 
    \beta_{Z,Y} 
    =
    \frac{\partial}{\partial z} \Exp{Y \mid x, w, z}.
\end{equation*}

In practice, when the structural equations are unknown and we have access to a dataset $\mathcal{D} = \{\mathbf{x}, \mathbf{w}, \mathbf{z}, \mathbf{y}\}$, we can create a sequential predictor using fair tuning (Algorithm \ref{alg:fair-tuning}). First, we obtain $\widehat{W}_\text{FT}$ using input $\mathcal{D} \setminus \mathbf{y}$ and the diagram $\diagram$ in Figure \ref{fig:indirect-effects-dgp}. Then we obtain $\yhat_\text{FT}$ using input $\mathcal{D}$ and $\diagram$, with the path $W \to Y$ indicated as an allowed path. Finally, we can construct the sequential predictor as 
\begin{equation*}
    \begin{split}
        \widehat{W}_\text{FT} &= 
        f_{\widehat{W}_\text{FT}}(Z) \\
        \yhat_\text{FT} &= f_{\yhat_\text{FT}}(\widehat{W}_\text{FT},Z).
    \end{split}
\end{equation*}

Sequential prediction should be preferred when the condition in Theorem \ref{thm:compatibility} is not assumed to hold for $Y,$ but to hold for a variable along the path, such as $W$ above. Another case when sequential prediction can be useful is if a parent along a not-allowed path has strong predictive power for $Y$, but is not deemed a business necessity. A limitation of sequential prediction using regression is that we need a variable to regress on. Hence, the presence of $Z$ in the mechanism of $W$ in the above example, where we regress $W$ on $Z$, is necessary.

\section{Time Complexity of Backpropagation with Fair Tuning Loss}
\label{app:time-complexity}
\FloatBarrier
Algorithm 2 details the computations necessary to backpropagate with the fair tuning loss (Equation \ref{eq:fair-tuning-loss}), with the difference to regular backpropagation highlighted in yellow. Consider a neural network architecture with one hidden layer with $h$ neurons and one output neuron. Then, backpropagation for $n$ training samples with $m$ features is of time complexity $O(n (m h + h))$.\footnote{Based on the scikit learn User Guide, \url{https://scikit-learn.org/stable/modules/neural_networks_supervised.html} accessed 23/02/2026.} Backpropagating with Equation \eqref{eq:fair-tuning-loss} adds the computation of the gradient loss (yellow highlights in Algorithm \ref{alg:fair-tuning-backprop}), which is of time complexity $O(n m)$. This is smaller than the time complexity of backpropagation, so fair tuning is also of $O(n (m h + h))$. However, since fair tuning involves more computations, specifically the gradient loss, the total time consumption will be larger. 
\begin{algorithm}
\caption{Backpropagation in fair tuning}
\label{alg:fair-tuning-backprop}
\begin{algorithmic}[1]
\Require Model $f_\theta$, feature values $\mathbf{x}$, outcomes $\mathbf{y}$, outcome gradients $\nabla\mathbf{y}$, loss $\mathcal{L} = \text{MSE}$, learning rate $\eta$

\State $\mathbf{\widehat{y}} \leftarrow f_\theta(\mathbf{x})$
\Comment{\textit{Generate predictions}}
\State $\mathcal{L}_{\text{pred}} \leftarrow \mathcal{L}(\mathbf{\widehat{y}},\, \mathbf{y})$
\Comment{\textit{Compute prediction loss}}

\vspace{0.5em}

\State
\colorbox{yellow!60}{
$\nabla \mathbf{\widehat{y}} \leftarrow \nabla_\mathbf{x} f_\theta(\mathbf{x})$
}
\Comment{\textit{Gradient computation}}

\State 
\colorbox{yellow!60}{
$\mathcal{L}_{\text{grad}} \leftarrow \mathcal{L}\!\left(\nabla\widehat{\mathbf{y}},\, \nabla \mathbf{y}\right)$
}
\Comment{\textit{Gradient loss}}

\vspace{0.5em}

\State $\mathcal{L}_{\text{total}} \leftarrow \mathcal{L}_{\text{pred}} + \mathcal{L}_{\text{grad}}$
\Comment{\textit{Sum losses}}
\State $\theta \leftarrow \theta - \eta \nabla_\theta \mathcal{L}_{\text{total}}$
\Comment{\textit{Update parameters}}
\end{algorithmic}
\end{algorithm}

We study the time complexity of backpropagating with and without Equation \eqref{eq:fair-tuning-loss} by simulation. We implement Algorithm \ref{alg:fair-tuning-backprop} in Python using PyTorch \citep{ansel_pytorch_2024} with Adam optimizer \citep{kingma_adam_2017} and execute it for 1000 replications for sample sizes $n$, number of features $m$, and number of hidden neurons $h$ in $\{100, 200, 300, 400, 500 \}$. The values in $\mathbf{x}$, $\mathbf{y}$, and $\nabla\mathbf{y}$ are randomly drawn from a normal distributions, $\text{N}(0,1)$. For each replicate, we time the computation of lines 1, 2, 5, 6 in Algorithm \ref{alg:fair-tuning-backprop} for backpropagation and lines 1-6 for fair tuning. The results were obtained on one core of a 11th Gen Intel(R) Core(TM) i7-1165G7 CPU at 2.80GHz, and the whole simulation study took 14 minutes and 23 seconds to run. 

The results are shown in Figure \ref{fig:time-complexity}, where each column corresponds to one of $n$, $m$, and $h$ on the x-axis, and each row corresponds to fixed values of the other variables, as indicated by the titles of each subplot. We see that the time consumption for both backpropagation and fair tuning scales linearly in $n$, $m$, and $h$, with fair tuning growing at a faster rate. 

\begin{figure}[ht!]
    \centering
    \includegraphics{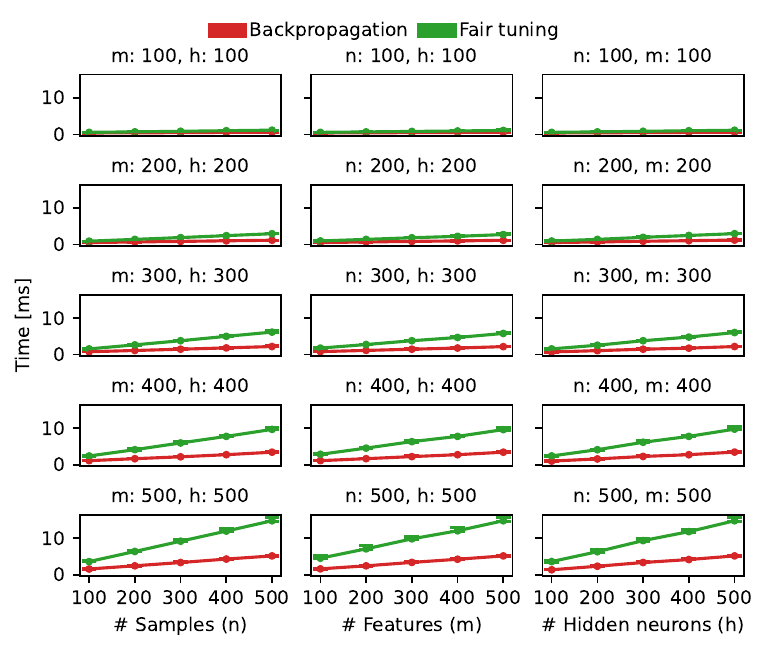}
    \caption{Time complexity of fair tuning and backpropagation, Algorithm \ref{alg:fair-tuning-backprop}. Each column corresponds to varying one variable ($n$, $m$, or $h$) along the x-axis, and each row corresponds to different, fixed values of the other two variables as indicated by subplot titles. Mean values and 95\% confidence intervals are based on 1000 replicates. The axes scales are the same for all subplots.}
    \label{fig:time-complexity}
\end{figure}

\newpage

\section{Results for All Models in Numerical Experiments}
\label{app:all-results}
\FloatBarrier

\begin{figure}[!ht]
    \centering
    \begin{subfigure}[b]{\textwidth}
        \centering
        \includegraphics{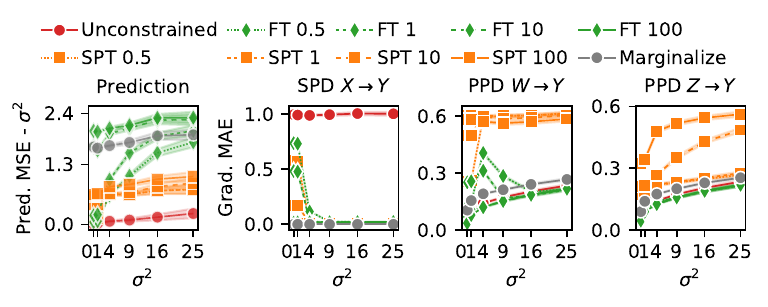}
        \caption{Linear setting}
        \label{fig:all-linear}
    \end{subfigure}
    \hfill
    \begin{subfigure}[b]{\textwidth}
        \centering
        \includegraphics{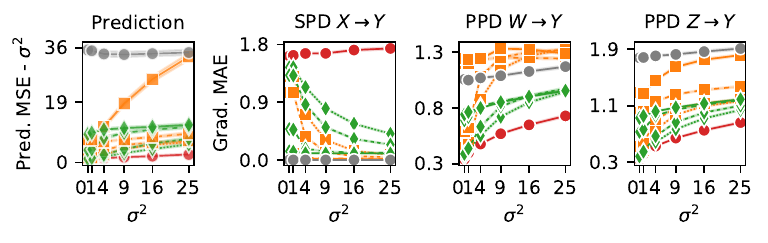}
        \caption{Multiplicative setting}
        \label{fig:all-multiplicative}
    \end{subfigure}
    \caption{Prediction loss (shifted by oracle MSE $\sigma^2$), SPD loss, and PPD loss for varying $\sigma^2$ values in (a) the Linear and (b) the Multiplicative setting. Lower is better for all losses. Results are averaged over 500 replicates with 95\% confidence intervals. SPT and FT predictors are indexed by their $\ls$ and $\ls=\lp$ tuning, respectively.}
    \label{fig:all-results}
\end{figure}

\FloatBarrier

\section{Mean Gradients of Predictors Trained on Linear DGP in Numerical Experiment}
\label{app:mean-gradient}

\FloatBarrier

 Figure \ref{fig:results-linear-mean-gradient} shows the mean gradient of predictors trained on data from the Linear DGP. Since the DGP is linear, the gradient is constant and the mean gradient of the predictors is a reasonable estimate of the gradient. The gradients of the Unconstrained predictor are close to the true gradient, $(-1, 1, 1)$. SPT 100's gradient along $X \to Y$ is increased from $-1$ to 0 due to the applied constraint while the gradients along $W \to Y$ and $Z \to Y$ are decreased compared to the Unconstrained predictor, from around 1 to around 0.5. This illustrates how the SPT predictor attributes variance to $W, Z$ when it is not allowed to attribute that variance to $X$ due to the constraint. This highlights the need for considering both statistical and predictive parity.

\begin{figure}[ht!]
    \centering
    \includegraphics{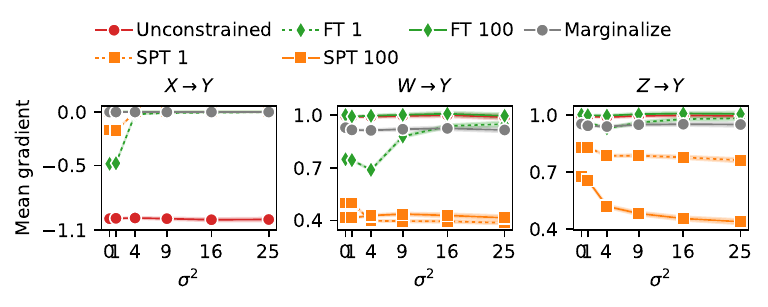}
    \caption{Mean gradients of different paths of several predictors trained on data from the Linear setting.}
    \label{fig:results-linear-mean-gradient}
\end{figure}

\FloatBarrier

\section{Results for All FT Predictors in Numerical Experiment}
\label{app:ft-results}
\begin{figure}[!ht]
    \centering
    \begin{subfigure}[b]{\textwidth}
        \centering
        \includegraphics{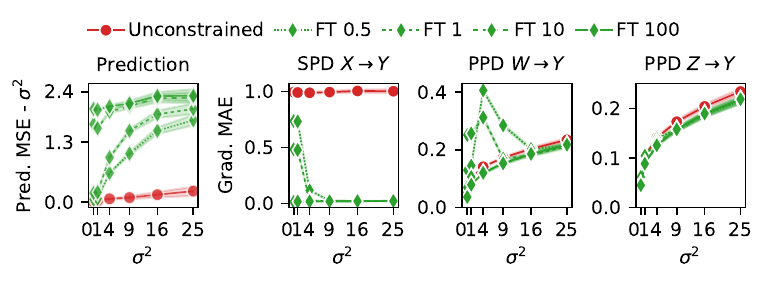}
        \caption{Linear setting}
        \label{fig:ft-only-linear}
    \end{subfigure}
    \hfill
    \begin{subfigure}[b]{\textwidth}
        \centering
        \includegraphics{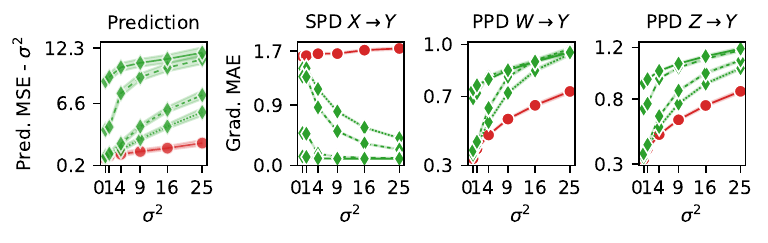}
        \caption{Multiplicative setting}
        \label{fig:ft-only-multiplicative}
    \end{subfigure}
    \caption{Prediction loss (shifted by oracle MSE $\sigma^2$), SPD loss, and PPD loss for varying $\sigma^2$ values in (a) the Linear and (b) the Multiplicative setting. Lower is better for all losses. Results are averaged over 500 replicates with 95\% confidence intervals. FT predictors are indexed by their$\ls=\lp$ tuning.}
    \label{fig:ft-only}
\end{figure}

\FloatBarrier

\section{Pareto Fronts in Numerical Experiments}
\label{app:pareto}

\FloatBarrier

A Pareto front shows solutions that are Pareto efficient for an optimization problem with more than one objective. In our case, a solution of tuning parameters $(\ls, \lp)$ is on the Pareto front if, for a given value of the SPD loss, the PPD is minimized.

Figure \ref{fig:pareto_numerical} shows the Pareto fronts for the numerical experiments when $\sigma = 0$ computed with $\ls, \lp$ linearly spaced from 0 to 100. In the Linear setting the Pareto front is close to $(0, 0)$, indicating the compatibility of SPD and PPD. In the Multiplicative case the Pareto front does not come close to zero, indicating that SPD and PPD is not compatible. In the Linear setting the Pareto front does not reach $(0,0)$ due to the estimation error of the true gradient which contributes to the PPD Loss. 

\begin{figure}[!ht]
    \centering
    \begin{subfigure}[b]{0.360\linewidth}
        \centering
        \includegraphics[width=\linewidth]{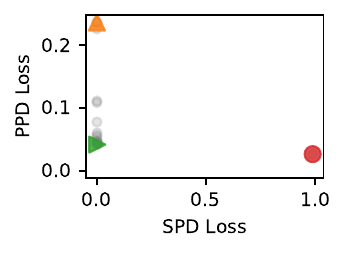}
        \caption{}
        \label{fig:pareto_linear}
    \end{subfigure}
    \hfill
    \begin{subfigure}[b]{0.56\linewidth}
        \centering
        \includegraphics[width=\linewidth]{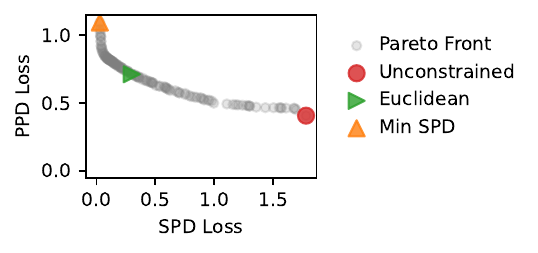}
        \caption{}
        \label{fig:pareto_multiplicative}
    \end{subfigure}
    \caption{The Pareto fronts for (a) the Linear and (b) the Multiplicative setting with $\sigma = 0$, based on 1024 linearly spaced values of $\ls$ and $\lp$ from 0 to 100. PPD loss is computed using the true gradient of the DGP. Unconstrained refers to the predictor with $\ls = \lp = 0$, Min SPD has the lowest SPD loss, Euclidean is the predictor closest to $(0,0)$ in Euclidean distance.}
    \label{fig:pareto_numerical}
\end{figure}

\FloatBarrier

\section{Pareto Front in COMPAS}
\label{app:compas-pareto}
\FloatBarrier
Figure \ref{fig:compas-pareto} shows the Pareto front of applying fair tuning on an unconstrained predictor trained on the COMPAS dataset. FT 10 achieves a good compromise between SPD and PPD losses, lying close to the $(0,0)$-point on the Pareto front.

\begin{figure}[!ht]
    \centering
    \includegraphics{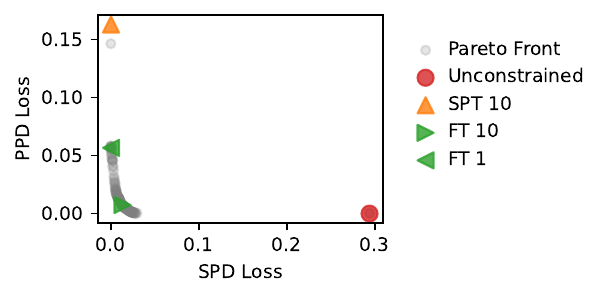}
    \caption{The Pareto front from 1024 predictors obtained over a $32 \times 32$ grid of $\lambda_s, \lambda_p$ values ranging from 0 to 10. Tuned predictors from Figure \ref{fig:compas-performance} are added and highlighted as stars.}
    \label{fig:compas-pareto}
\end{figure}

\FloatBarrier

\section{Tables of COMPAS results}
\label{app:compas-tables}

\FloatBarrier

\begin{table}[!ht]
\centering
\caption{Predictive (Accuracy, F1, AUC-ROC) and fairness (SPD, PPD Loss) performance on the COMPAS dataset. Mean values (first row) and confidence intervals (second row) are computed over 1000 bootstrap samples. Accuracy and F1 use a 0.5 threshold; AUC-ROC is the area under the ROC curve; PPD Loss is relative to the Unconstrained predictor. Higher is better for Accuracy, F1, AUC-ROC; lower is better for SPD and PPD Loss.}
\label{tab:compas-results}
\begin{tabular}{lrrrrr}
\toprule
Model           & Accuracy           & F1 Score           & AUC-ROC            & SPD Loss           & PPD Loss           \\ 
\midrule
Uncon.          &               0.69 &               0.65 &               0.75 &               0.60 &               0.00 \\ 
                &       (0.67, 0.70) &       (0.63, 0.68) &       (0.73, 0.76) &       (0.49, 0.72) &       (0.00, 0.00) \\[0.3em] 
FT 10           &               0.66 &               0.60 &               0.69 &               0.06 &               0.05 \\ 
                &       (0.64, 0.67) &       (0.56, 0.63) &       (0.67, 0.70) &       (0.03, 0.08) &       (0.03, 0.07) \\[0.3em] 
FT 1            &               0.65 &               0.58 &               0.69 &               0.00 &               0.22 \\ 
                &       (0.63, 0.66) &       (0.55, 0.63) &       (0.67, 0.71) &       (0.00, 0.01) &       (0.15, 0.30) \\[0.3em] 
SPT 10          &               0.64 &               0.59 &               0.69 &               0.00 &               0.38 \\ 
                &       (0.63, 0.66) &       (0.55, 0.63) &       (0.67, 0.71) &       (0.00, 0.00) &       (0.27, 0.49) \\[0.3em] 
Marg.           &               0.63 &               0.51 &               0.68 &               0.00 &               0.34 \\ 
                &       (0.62, 0.65) &       (0.46, 0.57) &       (0.67, 0.70) &       (0.00, 0.00) &       (0.22, 0.47) \\[0.3em] 
\bottomrule
\end{tabular}
\end{table}
\begin{table}[!ht]
\centering
\footnotesize
\caption{Comparison between derivative fairness (SPD and PPD) and contrast fairness (CSP and CPP) for the binary features of the COMPAS dataset. Mean values (first row) and 95\% confidence intervals (second row) are computed over 1000 bootstrap samples of the original dataset. PPD and CPP are computed w.r.t. the Unconstrained predictor.}
\label{tab:compas-comparison}
\begin{tabular}{lrrrrrr}
\toprule
Model           & SPD Race           & CSP Race           & SPD Sex            & CSP Sex            & PPD Degree         & CPP Degree         \\ 
\midrule
Uncon.          &               0.32 &               0.29 &               0.35 &               0.38 &               0.00 &               0.00 \\ 
                &       (0.19, 0.50) &       (0.18, 0.40) &       (0.23, 0.49) &       (0.25, 0.55) &       (0.00, 0.00) &       (0.00, 0.00) \\[0.3em] 
FT 10           &               0.02 &               0.08 &               0.01 &               0.11 &               0.02 &               0.11 \\ 
                &       (0.01, 0.03) &       (0.02, 0.23) &       (0.01, 0.03) &       (0.01, 0.38) &       (0.01, 0.03) &       (0.02, 0.24) \\[0.3em] 
FT 1            &               0.00 &               0.00 &               0.00 &               0.00 &               0.11 &               0.20 \\ 
                &       (0.00, 0.00) &       (0.00, 0.00) &       (0.00, 0.00) &       (0.00, 0.00) &       (0.05, 0.20) &       (0.09, 0.33) \\[0.3em] 
SPT 10          &               0.00 &               0.00 &               0.00 &               0.00 &               0.27 &               0.25 \\ 
                &       (0.00, 0.00) &       (0.00, 0.00) &       (0.00, 0.00) &       (0.00, 0.00) &       (0.18, 0.40) &       (0.18, 0.35) \\[0.3em] 
Marg.           &               0.00 &               0.00 &               0.00 &               0.00 &               0.19 &               0.19 \\ 
                &       (0.00, 0.00) &       (0.00, 0.00) &       (0.00, 0.00) &       (0.00, 0.00) &       (0.09, 0.33) &       (0.09, 0.34) \\[0.3em] 
\bottomrule
\end{tabular}
\end{table}

\FloatBarrier

\bibliography{sn-bibliography}

\end{document}